\documentclass[preprint,12pt]{elsarticle}

\usepackage{lipsum}
\usepackage{amsfonts}
\usepackage{algorithmic}
\usepackage{algorithm} 
\usepackage{float} 
\usepackage{xcolor}
\usepackage{ifpdf} 

\ifpdf
  \DeclareGraphicsExtensions{.eps,.pdf,.png,.jpg}
\else
  \DeclareGraphicsExtensions{.eps}
\fi

\usepackage{amsthm}
\newtheorem{theorem}{Theorem}
\newtheorem{lemma}{Lemma}
\newtheorem{corollary}{Corollary}

\newtheorem{remark}{Remark}

\usepackage{amsopn}

\usepackage{amsmath,amsfonts,bm}

\def\eqref#1{equation~\ref{#1}}

\def\1{\bm{1}}

\def\ve{{\bm{e}}}
\def\vf{{\bm{f}}}

\def\vk{{\bm{k}}}

\def\vr{{\bm{r}}}

\def\vu{{\bm{u}}}

\def\vx{{\bm{x}}}

\DeclareMathAlphabet{\mathsfit}{\encodingdefault}{\sfdefault}{m}{sl}
\SetMathAlphabet{\mathsfit}{bold}{\encodingdefault}{\sfdefault}{bx}{n}

\newtheorem{question}{Question}[section]
\usepackage{multirow}
\usepackage{booktabs}
\usepackage{graphicx,epstopdf}
\usepackage[caption=false]{subfig}
\graphicspath{{./}}
\usepackage[hidelinks]{hyperref}
\usepackage{cleveref}
\crefname{hypothesis}{Hypothesis}{Hypotheses}
\crefname{fact}{Fact}{Facts}


\usepackage{tikz}
\usetikzlibrary{shapes.geometric, arrows.meta, positioning, fit, calc, shadows}

\journal{Preprint}

\begin{document}
\begin{frontmatter}

\title{Self-composing neural operators for high-frequency and multiscale PDE surrogates}
\author[1]{Juncai He}
\ead{jche@tsinghua.edu.cn}

\author[2,4]{Xinliang Liu\corref{cor1}}
\ead{xinliang.liu@ouc.edu.cn}
\ead{xinliang.liu.slai@gmail.com}

\author[3]{Jinchao Xu}
\ead{jinchao.xu@kaust.edu.sa}

\cortext[cor1]{Corresponding author}

\address[1]{Yau Mathematical Sciences Center, Tsinghua University, Haidian District, Beijing 100084, China}
\address[2]{Shenzhen Loop Area Institute}
\address[3]{King Abdullah University of Science and Technology, Thuwal, Saudi Arabia}
\address[4]{Ocean University of China, Qingdao, Shandong 266100, China}

\begin{abstract}
Addressing the computational challenges of high-frequency and multiscale partial differential equations (PDEs), this work introduces a self-composing neural operator (SC-NO) framework. Inspired by classical fixed-point iterative solvers (e.g., multigrid, domain decomposition), the proposed architecture constructs a deep operator by repeatedly applying a single, parameter-efficient backbone block. This design mimics the update step of a numerical solver, allowing the model to progressively resolve complex solution features without increasing the parameter count. For practical training, we develop an adaptive ``Train-and-Unroll'' strategy that grows the composition depth during training, acting as a curriculum from shallow to deep self-composed models. We demonstrate the efficacy of this framework on the Helmholtz equation for ultrasound computed tomography (USCT), a problem characterized by high-frequency wave propagation in highly heterogeneous media. By instantiating the backbone with a multigrid-inspired architecture, the SC-NO effectively mitigates the spectral bias often observed in standard operator learning baselines. Numerical experiments show that our method reduces the prediction error significantly compared to Fourier Neural Operators (FNO) and their variants in the 300--500 kHz regime. Furthermore, we provide theoretical analysis linking the self-composition depth to approximation accuracy.
\end{abstract}


\end{frontmatter}

\section{Introduction}\label{sec:intro}
Partial differential equation (PDE) models are ubiquitous in physics, engineering, and other disciplines. Many scientific and engineering fields rely on solving PDEs, such as optimizing airfoil shapes for better airflow, predicting weather patterns by simulating the atmosphere, and testing the strength of structures in civil engineering. Tremendous efforts have been made to solve various PDEs arising from different areas. Solving PDEs usually requires designing numerical methods that are tailored to the specific problem and depend on the insights of the model. However, deep learning models have shown great promise in solving PDEs in a more general and efficient way.
Recently, several novel methods have been developed to directly learn the operator (mapping) between infinite-dimensional parameter and solution spaces of PDEs. Notable examples include DeepONet~\cite{lu2021learning}, which employs a branch-trunk architecture grounded in universal approximation theory~\cite{chen1995universal}, and the Fourier Neural Operator (FNO)~\cite{li2020fourier}, which parameterizes global convolutional operators using a fast Fourier transform. This line of research has expanded rapidly, yielding variants like geo-FNO~\cite{li2022fourier} for complex geometries, the Latent Spectral Model (LSM)~\cite{wu2023solving} that operates in a learned latent space, and Multiwavelet-based\break Transformers (MWT)~\cite{gupta2021multiwavelet} using wavelet transforms. Concurrently, attention mechanisms, popularized by Transformers~\cite{vaswani2017attention}, were adapted for operator learning, as seen in Galerkin-type attention models~\cite{cao2021choose} and the General Neural Operator Transformer (GNOT)~\cite{hao2023gnot}, which introduces heterogeneous normalized attention for handling multiple input functions and irregular meshes. More recently, explicitly convolution-based architectures have also been proposed as neural operators~\cite{raonic2023convolutional}, and scalable approaches like Transolver++~\cite{luo2025transolver} have demonstrated the ability to handle million-scale geometries for industrial applications. Additionally, specialized frameworks for complex geometries include Point Cloud Neural Operators (PCNO)~\cite{zeng2025point}, which efficiently handle variable domains with topological variations. MIONet is incorporated into classical iterative solvers, enhancing their efficiency \cite{jin2022mionet, hu2025hybrid}. The NINO~\cite{hao2024newton} has demonstrated how neural operators can learn to integrate traditional numerical techniques with Newton nonlinear solvers, effectively learning the nonlinear mapping at each iteration for multiple solutions.

These approaches can be broadly categorized. Spectral-type neural operators (e.g., FNO, MWT, LSM) leverage global operations and have shown promising results, supported by universal approximation theorems~\cite{kovachki2023neural,lanthaler2023nonlocal}. However, they can exhibit a spectral bias toward low-frequency components, while Fourier-mode truncation can further constrain high-frequency resolution~\cite{xu2025spectralbias}. MscaleFNO addresses this issue with parallel FNO branches and scaled inputs for oscillatory mappings and Helmholtz wave scattering~\cite{you2026mscalefno}. Local and geometrical neural networks (e.g., ResNet~\cite{he2016deep, he2016identity}, UNet~\cite{ronneberger2015u}) offer greater flexibility in managing diverse boundary conditions. Yet, these were not originally designed for operator learning of PDEs and often require a substantial number of parameters to achieve high accuracy. They can also act as high-pass filters, focusing on local features, which may limit their ability to capture long-distance dependencies crucial in many physical phenomena, unless specific mechanisms like large kernels are incorporated, which can be hard to train~\cite{ding2022scaling}. {A nearby but distinct class of PDE-informed learning methods embeds more of the governing equations into the learning procedure. In fluid mechanics, physics-constrained deep learning has been used to construct simulation-free surrogates, reconstruct flow fields from sparse and noisy observations, handle irregular geometries with geometry-adaptive CNNs, and improve flow-field resolution without high-resolution labels~\cite{sun2020surrogate,sun2020physics,gao2021phygeonet,gao2020super}. These works are relevant because they use physical information such as PDE residuals, boundary conditions, or discretized operators, but they are not direct amortized neural-operator baselines for the setting considered here. Related residual-driven methods integrate learned components directly into classical solver iterations, including scattering surrogates such as SwitchNet~\cite{khoo2019switchnet}, neural-augmented relaxation and preconditioning strategies~\cite{zhang2024blending,kopanicakova2025deeponet}, and multigrid-augmented Helmholtz preconditioners~\cite{azulay2023multigrid,lerer2024compact,cui2025neuralmg}. These methods are important references for our problem class, but they assume access to the discretized operator \cite{liu2024multi}, residual evaluation, or an iterative solver loop at inference time. In contrast, our goal is amortized operator surrogacy: after offline training, the model maps coefficient and source fields directly to the solution field without assembling the PDE operator or executing Krylov or multigrid iterations online. A detailed treatment of residual-driven neural preconditioners is therefore a separate solver-design topic rather than the focus of the present surrogate-learning manuscript.} Physics-aware operator architectures can instead encode structural constraints in the learned surrogate: the Peridynamic Neural Operator learns a nonlocal constitutive law within a state-based peridynamic formulation while preserving objectivity and balance laws~\cite{jafarzadeh2024peridynamic}. Other extensions target posterior inference or predictive uncertainty rather than only forward-surrogate accuracy. Gao et al. adaptively refine an operator surrogate during posterior evaluation for infinite-dimensional Bayesian inverse problems, while IB-UQ predicts both means and uncertainties for neural function regression and neural operator learning~\cite{gao2024adaptive,guo2024ibuq}. These objectives are complementary to, but are not evaluated in, the present work.  Applications in biomechanical engineering and weather forecasting \cite{you2022physics, pathak2022fourcastnet} further underscore the need for architectures that can combine the representational power of deep models with the computational parsimony of classical solvers.

Diffusion-augmented neural operators provide another route by improving spectral representation through a generative refinement stage, which is distinct from the deterministic forward-surrogate architecture considered here~\cite{oommen2025diffusion}.

Motivated by the iterative nature of many classical PDE solvers, where a conventional procedure is repeated to progressively refine a solution, we pose the following question:
\begin{question}
Can we design and efficiently train an appropriate neural operator such that it repeats only a simple skeleton structure (low-cost) but can achieve high accuracy through this repetition?
\end{question}
This question draws inspiration from some of the most powerful techniques in scientific computing. Among all iterative methods for numerical PDEs, multigrid methods~\cite{hackbusch2013multi,xu1989theory,trottenberg2000multigrid} are renowned for their efficiency, especially for elliptic PDEs. The connection between multigrid and deep learning was first noted in the original ResNet paper~\cite{he2016deep}, citing multigrid as a rationale for residual connections. Subsequently, MgNet~\cite{he2019mgnet,he2023interpretive} established deeper links, demonstrating that a linear V-cycle multigrid for the Poisson equation could be represented as a CNN. While MgNet and its adaptations have been explored for numerical PDEs~\cite{chen2022meta,zhu2023enhanced} and forecasting~\cite{zhu2023fv}, a general and effective framework for integrating multigrid principles directly into operator learning for a broader class of PDEs has remained an open area. {In this work, we aim to fundamentally integrate multigrid methodologies with operator learning based on a very general framework.}

To address these challenges, particularly those associated with high-frequency and multiscale PDEs, we investigate the intrinsic properties of operator learning tasks governed by PDEs and propose a concise and elegant neural operator architecture based on the concept of self-composition, directly inspired by iterative numerical methods. We introduce the general self-composing formulation $\mathcal O (\bm v) = \mathcal P \circ (\mathcal G \circ)^n \circ \mathcal L (\bm v)$ for neural operators. This framework allows building deep operators with significant parameter sharing by repeatedly applying a backbone operator $\mathcal G(\cdot)$.

The concept of repeatedly applying a computational block shares similarities with several existing paradigms. Classical iterative methods for PDEs, such as Jacobi or Gauss-Seidel, inherently involve repeating a fixed update rule. Recurrent Neural Networks (RNNs)~\cite{connor1994recurrent} also apply the same transition function repeatedly over a sequence. Algorithm unrolling~\cite{gregor2010learning} explicitly maps iterations of an optimization algorithm onto layers of a deep network.
Our self-composing framework is distinct in several ways. While motivated by iterative PDE solvers, it focuses on learning the solution operator itself. Unlike standard RNNs which process sequential data, our model typically takes static function data as input and is primarily concerned with the final output after $n$ compositions. Compared to algorithmic unrolling, we share a single backbone $\mathcal G$ across all compositions (tied parameters) rather than learning stage-specific parameters; compared to deep equilibrium models, we avoid solving an implicit fixed point and instead train at adaptively increasing explicit depths with the same $\mathcal G$ (Section~\ref{sec:train}, Fig.~\ref{fig:train_and_unroll_diagram}). This shared-$\mathcal G$ design promotes parameter efficiency, while adaptive depth growth improves computational efficiency and gradient stability during training.

Another perspective in deep learning involves viewing networks as discretized ordinary differential equations (ODEs)~\cite{e2017a,haber2018learning,lu2018beyond,chen2018neural}, where network depth corresponds to time steps. Pushing this to infinite depth leads to implicit neural network models like Deep Equilibrium Models (DEQs)~\cite{bai2019deep,geng2021training,fung2022jfb}, recently extended to implicit neural operators~\cite{marwah2023deep}. In these models, the output is found by solving for an equilibrium point. While elegant,  the accuracy cannot be readily improved by further self-composition in the explicit, iterative manner our framework allows.

In this work, we introduce a novel general self-composing formulation for neural operators, defined as $\mathcal{O}(\bm{v}) = \mathcal{P} \circ (\mathcal{G} \circ)^n \circ \mathcal{L}(\bm{v})$, directly inspired by iterative numerical methods commonly employed in PDE solvers, which is specifically designed to handle high-frequency and multiscale features effectively. This formulation facilitates the construction of deep neural operators through extensive parameter sharing by repeatedly applying a single backbone operator $\mathcal{G}(\cdot)$.
We provide robust theoretical justification for our self-composing structure, including qualitative universal approximation results (see Theorem~\ref{them:approxi}) and quantitative error reduction guarantees as the composition depth $n$ increases (see Theorem~\ref{them:rate}). Additionally, we introduce an efficient adaptive training technique, termed the Train-and-Unroll strategy, which progressively increases the composition depth $n$ throughout training, effectively leveraging previously learned weights from shallower network structures.
To empirically demonstrate the parameter efficiency and effectiveness of our proposed self-composing structure, we apply it to benchmark Darcy flow problems, where our model exhibits significant advantages. Moreover, we further specialize our framework by designing an MgNet-inspired backbone augmented with an Adaptive Convolution Mechanism tailored specifically for the challenging Helmholtz equation arising in ultrasound computed tomography (USCT). Quantitatively, our approach attains similar Darcy errors with substantially fewer parameters (Table~\ref{tab:darcy}), and on USCT reduces RRMSE by 9–13× versus strong baselines at 300–500 kHz (Table~\ref{tab:forward-baselines}).

The remainder of this paper is structured as follows. Section~\ref{sec:IntroNO} details the architecture of the self-composing neural operators, providing both qualitative and quantitative theoretical justifications alongside numerical experiments. Section~\ref{sec:train} introduces the Train-and-Unroll strategy, our dynamic depth training algorithm for these operators. In Section~\ref{sec:backbone}, we describe a specific MgNet-inspired backbone designed for the Helmholtz equation and present numerical results that demonstrate the efficiency and robustness of the proposed method. Finally, concluding remarks and future directions are discussed in Section~\ref{sec:conclusions}.

\section*{Notation summary}
We summarize key symbols used in the manuscript.
\begin{center}
\begin{tabular}{ll}
\toprule
Symbol & Meaning \\
\midrule
$n$ & Composition depth (number of self-compositions) \\
$L$ & Backbone depth (hidden layers within $\mathcal G$) \\
$N$ & Constant width across layers when $N_\ell\equiv N$ \\
$N_\ell$ & Width of the $\ell$-th hidden layer \\
$\widetilde N$ & Latent width for lift/project spaces $\mathcal Z^{\widetilde N}$ \\
$\mathcal L,\mathcal P$ & Lift and projection operators \\
$\mathcal G$ & Shared backbone operator (tied across compositions) \\
$\mathcal O$ & Self-composed operator $\mathcal P\circ(\mathcal G\circ)^n\circ\mathcal L$ \\
\bottomrule
\end{tabular}
\end{center}

\section{Self-composing neural operators and theoretical analysis}\label{sec:IntroNO}
In this section, we first introduce an abstract deep (fully connected) neural operator framework for mapping between (vector-valued) Hilbert function spaces, which can be generalized to Banach spaces. Then, we propose the self-composing neural operator architecture. Furthermore, we prove the universal approximation theory for this novel neural operator with some further theoretical and practical remarks and properties. 

\subsection{Abstract deep neural operators}
Let $\mathcal Z = H^s(\Omega)$ denote a Hilbert space over a bounded domain $\Omega \subset \mathbb R^d$ with the notation for the product space 
$$
\mathcal Z^n := \underbrace{\mathcal Z \otimes \mathcal Z \otimes \cdots \otimes \mathcal Z}_{n}.
$$
Following the notation in~\cite{lanthaler2023nonlocal,he2023mgno}, we introduce the abstract deep neural operator framework $\mathcal O: \mathcal Z^{\tilde d} \mapsto \mathcal Z^{\Bar d}$ with a fully connected structure regarding neurons. To begin, we denote the bounded linear operator acting between these product spaces as $\mathcal W \in \mathcal L\left(\mathcal Z^n, \mathcal Z^m\right)$. 
To be more specific, the relation $\left[ \mathcal W \bm h \right]_i = \sum_{j=1}^n \mathcal W_{ij} \bm h_j$ holds for $ i=1:m$, 
where $\bm h_j \in \mathcal Z$ and $\mathcal W_{ij} \in \mathcal L\left(\mathcal Z, \mathcal Z\right)$.
Now, let us introduce $N_\ell \in \mathbb N^{+}$ for all $\ell=1:L$ as the number of neurons in the $\ell$-th hidden layer with $N_{0} = \tilde d$ and $N_{L+1} = \Bar d$.
Then, the deep neural operator $\mathcal O(\bm v)$ with $L$ hidden layers and $N_\ell$ neurons in the $\ell$-th layer is defined as
\begin{equation}\label{eq:dnodef}
    \begin{cases}
        &\bm h^0(\bm v) = \bm v \in \mathcal Z^{N_0} \\
        &\bm h^{\ell}(\bm v) = \sigma\left(\mathcal W^\ell \bm h^{\ell-1}(\bm v) + \mathcal B^\ell \right) \in \mathcal Z^{N_\ell}  \quad \ell=1:L \\
        &\mathcal O(\bm v) = \mathcal W^{L+1} \bm h^L(\bm v) + \mathcal B^{L+1} \in \mathcal Z^{N_{L+1}}
    \end{cases}
\end{equation}
where $\mathcal W^{\ell} \in \mathcal L\left(\mathcal Z^{N_{\ell-1}}, \mathcal Z^{N_{\ell}}\right)$ with $\mathcal B^\ell \in \mathcal Z^{N_\ell}$
and $\sigma: \mathbb R \mapsto \mathbb R$ defines the nonlinear point-wise activation. Unless otherwise specified, we reserve $n$ for composition depth and use $N$ for width; for constant width we set $N_\ell = N$ for all $\ell=1:L$.

As shown in \cite{he2023mgno}, although the above fully connected architecture is quite straightforward, it provides a very general framework for most existing neural operators. Basically, they mainly focus on how to parameterize the bounded linear operator $[\mathcal W^\ell]_{ij}$, for instance DeepONet~\cite{lu2021learning}, Fourier NO~\cite{li2020fourier}, Low-rank NO~\cite{kovachki2023neural}, Wavelet NO~\cite{gupta2021multiwavelet}, Riemannian/Laplacian NO~\cite{chen2023learning}, Nonlocal NO~\cite{lanthaler2023nonlocal}, MgNO~\cite{he2023mgno}, etc.. 
In our work, we also propose a new parameterization mechanism with a natural motivation which will be introduced in the next section. Different from others, we also propose a new key architecture for building efficient and accurate deep neural operators motivated by iterative methods in solving numerical PDEs.

\subsection{Self-composing neural operators motivated by iterative methods in solving PDEs}
Typically, a standard PDE problem can be formulated as follows: Given a coefficient or parameter function $\bm{v} \in \mathcal{X}^{d_{\rm in}}$, find the solution $\bm{u} \in \mathcal{Y}^{d_{\rm out}}$ such that
\begin{equation}
\mathcal{D}(\bm{u}; \bm{v}) = 0,
\end{equation}
where $\mathcal{D}$ denotes the specific PDE operator involving derivatives and boundary information. Here, $\mathcal{X}$ and $\mathcal{Y}$ are Sobolev spaces defined on a bounded domain $\Omega \subset \mathbb{R}^d$, and $d_{\rm in}$ and $d_{\rm out}$ account for the diversity of different types of PDE problems. The operator learning task, therefore, is to approximate the operator $\mathcal{O}^*: \mathcal{X}^{d_{\rm in}} \to \mathcal{Y}^{d_{\rm out}}$.
In solving PDEs (both theoretically or numerically), we typically apply some iterative methods with certain initialization to approximate the solution as
\begin{equation}
    \bm u^{n} = \mathcal I(\bm u^{n-1};\bm v).
\end{equation}
To express this iterative process as a self-composition, we define an augmented operator $\widetilde{\mathcal I}:\mathcal Y^{d_{\rm out}} \times \mathcal X^{d_{\rm in}} \mapsto \mathcal Y^{d_{\rm out}} \times \mathcal X^{d_{\rm in}}$ whose action on the augmented state $(\bm u^{n-1}, \bm v)$ is given by
\begin{equation}
    \widetilde{\mathcal I} \begin{pmatrix} \bm u^{n-1} \\ \bm v \end{pmatrix} := \begin{pmatrix} \mathcal I(\bm u^{n-1}; \bm v) \\ \bm v \end{pmatrix} = \begin{pmatrix} \bm u^n \\ \bm v \end{pmatrix}.
\end{equation}
By applying this operator $n$ times to the initial state $(\bm u^0, \bm v)$, we obtain the state at the $n$-th iteration:
\begin{equation}
    \begin{pmatrix}
        \bm u^n \\ \bm v
    \end{pmatrix} =  (\widetilde{\mathcal I} \circ)^n \begin{pmatrix}
        \bm u^0 \\ \bm v
    \end{pmatrix},
\end{equation}
where $(\widetilde{\mathcal I} \circ)^n$ denotes the self-composition of $\widetilde{\mathcal I}$ for $n$ times.  
If the iterative method $\mathcal I$ is strictly monotonically convergent, we have $e_n < e_m$ for any $m < n$ where
\begin{equation}
    e_n  := \left\|
    \begin{pmatrix}
        \bm u \\ \bm v
    \end{pmatrix} - (\widetilde{\mathcal I} \circ)^n \begin{pmatrix}
        \bm u^0 \\ \bm v
    \end{pmatrix} \right\|_{\mathcal Y^{d_{\rm out}} \times \mathcal X^{d_{\rm in}}} = \left\|\bm u - \bm u^n\right\|_{\mathcal Y^{d_{\rm out}}}
\end{equation}
Furthermore, for any linearly convergent  iterative method $\mathcal I$, there exists $0 < \delta < 1$ such that $e_n \le \delta^n e_0$. Motivated by this simple structure of self-composition and dynamic error decaying property, we propose the following fully connected deep neural operators with self-composing structure
\begin{equation}\label{eq:scnodef}
    \mathcal O (\bm v) = \mathcal P \circ (\mathcal G \circ)^n \circ \mathcal L (\bm v)
\end{equation}
where $\mathcal L: \mathcal X^{d_{\rm in}} \to \mathcal Z^{\widetilde N}$ and $\mathcal P: \mathcal Z^{\widetilde N} \to \mathcal Y^{d_{\rm out}}$ are two bounded linear operators for some $\widetilde N \in \mathbb N^+$ and $\mathcal G: \mathcal Z^{\widetilde N} \to \mathcal Z^{\widetilde N}$ defines a vanilla deep fully connected neural operator with a fixed depth $L\in\mathbb N^+$, width $N\in\mathbb N^+$ and activation $\sigma(\cdot)$ as in \eqref{eq:dnodef} (with $n_0 = n_{L+1} = \widetilde N$ because of the existing $\mathcal P$ and $\mathcal L$). The components of this architecture parallel the augmented iterative method: the lifting operator $\mathcal{L}$ creates an initial latent state from the input $\bm{v}$, analogous to creating the initial augmented state $(\bm{u}^0, \bm{v})$. The backbone operator $\mathcal{G}$ acts as the neural analogue of the iterative update operator $\widetilde{\mathcal{I}}$, and its repeated self-composition for $n$ times mimics the iterative process. Finally, the projection operator $\mathcal{P}$ extracts the final solution from the latent space, similar to retrieving $\bm{u}^n$ from the final augmented state.

If $n=1$, the neural operator defined in \eqref{eq:scnodef} degenerates to the standard deep fully connected neural operator as in \eqref{eq:dnodef}. When $n\ge2$, this self-composing neural operator architecture is quite different from the standard one. 
If the depth of $\mathcal G$ is $L$, then one can also interpret $\mathcal P \circ (\mathcal G \circ)^n \circ \mathcal L $ as a fully connected neural operator with $n\times L$ hidden layers but sharing parameters along all $n$ blocks where each block consists of a deep neural operator with $L$ hidden layers. 
That is, the self-composing neural operator can efficiently save the memory cost for building a very deep neural operator. 
Given this special structure and efficiency in memory, we surprisingly still have the following universal approximation theory for self-composing neural operators.

\subsection{Theoretical justification for self-composing neural operators}
For simplicity, let us assume $d_{\rm in} = d_{\rm out} = 1$ and $\mathcal O^*: \mathcal C \subset \mathcal X \to \mathcal Y$ is a continuous operator on a compact domain $\mathcal C$. 
We first present the following universal approximation result for self-composing neural operators with any fixed $n$.
\begin{theorem}\label{them:approxi}
For any $n\in \mathbb N^+$ and $\epsilon > 0 $ , there exist $N, \widetilde N, L \in \mathbb N^+$ with two bounded linear operators
$\mathcal L: \mathcal X \to \mathcal Z^{\widetilde N}$ and $\mathcal P: \mathcal Z^{\widetilde N} \to \mathcal Y$ and a neural operator $\mathcal G: \mathcal Z^{\widetilde N} \to \mathcal Z^{\widetilde N}$ with ReLU activation function, $L$ hidden layers and $N$ neurons in each layer such that
\begin{equation}
    \sup_{\bm v \in \mathcal C}\left\| \mathcal O^*(\bm v) - \mathcal P \circ (\mathcal G \circ)^n \circ \mathcal L (\bm v)\right\|_{\mathcal X} \le \epsilon.
\end{equation}
Here, $\mathcal Z$ can be any function spaces that contain the constant function $\bm 1(x)$ on $\Omega$.
\end{theorem}
\begin{remark}
Assume throughout that $\mathcal C\subset \mathcal X$ is compact. The theorem asserts that for any fixed composition depth $n$, there exist finite width $N$, latent width $\widetilde N$, and backbone depth $L$ such that the tied-parameter self-composition approximates $\mathcal O^*$ uniformly on $\mathcal C$. Practically, one can fix parameters of $\mathcal G$ and still control accuracy by adjusting the composition depth $n$.
\end{remark}
Before we prove the main theorem, we present the next important lemma about the approximation of a family of high-dimensional functions by the self-composition structure of deep ReLU neural networks.
\begin{lemma}\label{lem:NNAprox}
	For any $\epsilon > 0$, $n \in \mathbb N^+$, and a sequence of constant numbers $\left\{c_i\right\}_{i=1}^m$ and functions $\left\{F_i(x)\right\}_{i=1}^m$ on $[0,1]^k$, there exists a sequence of ReLU DNNs $\left\{\widetilde F_i(x)\right\}$ such that
	\begin{equation*}
		\left\|F_i - \widetilde F_i\right\|_{L^{\infty}([0,1]^k)} \le \frac{\epsilon}{3mc_i}
	\end{equation*}
	for each $i=1:m$ with 
	$$
	\widetilde F_i(x) = P_i \circ (g_i \circ)^{n} \circ L_i (x)
	$$
	where $L_i: \mathbb R^k \mapsto \mathbb R^{\widetilde k}$ and $P_i: \mathbb R^{\widetilde k} \mapsto \mathbb R$ are affine mappings and $g_i: \mathbb R^{\widetilde k} \mapsto\mathbb R^{\widetilde k} $ is a deep ReLU neural network function with $L$ hidden layers and $\overline k_i$ neurons in each layer.
\end{lemma}

\begin{proof}
	Given the approximation result in Theorem 1.3 in \cite{zhang2023enhancing}, for any $F_i(x)$ and $n_i \in \mathbb N^+$ with $i=1:m$, there exist ReLU DNNs $\hat F_i(x)$ such that
	\begin{equation}    
		\|F_i - \hat F_i\|_{L^\infty([0,1]^k)} \le 6 \sqrt{k} \omega_{F_i} \left(n_i^{-1/k}\right) \le \frac{\epsilon}{3mc_i},
	\end{equation}
	with $\hat F_i(x) =  P_i \circ (\hat g_i\circ)^{n_i} \circ  L_i (x)$
	where $ P_i: \mathbb R^{\hat k} \mapsto \mathbb R$ and $ L_i (x): \mathbb R^k \mapsto \mathbb R^{\hat k}$ are linear maps with $\hat k = 3^k(5k+4)-1$ and $\hat g_i: \mathbb R^{\hat k} \mapsto \mathbb R^{\hat k}$ are ReLU DNNs with $4^{k+5}k$ neurons and $3+2k$ hidden layers.
	Here, $\omega_{f}(r) := \sup_{\|x-y\|\le r}|f(x) - f(y)|$ denotes the modulus of continuous function $f(x)$ on $[0,1]^k$.
	
	Thus, for all $\left\{\frac{\epsilon}{3mc_i}\right\}_{i=1}^m$, one can take $\overline n = qn \ge \max_{i}\left\{n_i\right\}$ for some $q \in \mathbb N^+$.
	Consequently, calling Theorem 1.3 in \cite{zhang2023enhancing} again, we have ReLU DNNs $\overline F_i$ such that 
	\begin{equation}
		\|F_i - \overline F_i\|_{L^\infty([0,1]^k)} \le 6 \sqrt{k} \omega_{F_i} \left(\overline n_i^{-1/k}\right)  \le 6 \sqrt{k} \omega_{F_i} \left( n_i^{-1/k}\right) 
		\le \frac{\epsilon}{3mc_i}
	\end{equation}
	for all $i=1:m$ with $\overline F_i(x) =  \overline P_i \circ (\overline g_i\circ)^{n_i} \overline  L_i (x)$
	where $\overline P_i: \mathbb R^{\overline k} \mapsto \mathbb R$ and $\overline L_i (x): \mathbb R^k \mapsto \mathbb R^{\overline k}$ are linear maps with $\overline k = \hat k = 3^k(5k+4)-1$ and $\overline g_i: \mathbb R^{\hat k} \mapsto \mathbb R^{\hat k}$ are ReLU DNNs with $4^{k+5}k$ neurons and $3+2k$ hidden layers.
	
	Finally, we can finish the proof by taking $g_i = (\overline g_i\circ)^{q}$, 
	$P_i = \overline P_i$, and $L_i = \overline L_i$ for all $i=1:m$.
\end{proof}

With Lemma~\ref{lem:NNAprox} on hand, we can present the proof for Theorem~\ref{them:approxi}. 
\begin{proof}
	If $n=1$,  it degenerates to the classical approximation theorem of neural operators with only one hidden layer or a general deep neural operator architecture, which can be found in ~\cite{lanthaler2023nonlocal,he2023mgno}.
	
	Here, we are more interested in the case $n\ge 2$ since this requires a repeated composition of a fixed ReLU neural network. We split the proof to the following steps.
	
	\paragraph{Finite-dimension approximation of $\mathcal O^*$ by projection:} 
	Since $\mathcal C \subset \mathcal X$ is a compact set and $\mathcal O^*$ is continuous, we have the $\mathcal O^*(\mathcal C)$ is also compact in $\mathcal Y$. Thus, for any $\epsilon > 0$, there are unit orthogonal basis function $\{\phi_1, \cdots, \phi_m \} \subset \mathcal Y$ and continuous functionals $f_i: \mathcal X \mapsto \mathbb R$ for $i=1:m$ such that
	\begin{equation}
		\sup_{\bm v\in \mathcal C} \left\| \mathcal O^*(\bm v) - \sum_{i=1}^m f_i(\bm v) \phi_i \right\|_{\mathcal Y} \le \frac{\epsilon}{3}.
	\end{equation}
	Thus, we only need to prove that there is deep neural operator $\mathcal O_i = \mathcal P_i \circ (\mathcal G_i\circ)^{n} \circ \mathcal L_i$ such that
	\begin{equation}
		\sup_{\bm v\in \mathcal C} \left\| f_i(\bm v) \phi_i - \mathcal O_i(\bm v)\right\|_{\mathcal Y} \le \frac{2\epsilon}{3m}.
	\end{equation}
	
	\paragraph{Parameterization (approximation) of $\mathcal X$ with finite dimensions to discretize $f_i$:}
	Since $\mathcal X = H^{s}(\Omega)$ and $\mathcal C$ is compact, we can find $k \in \mathbb N^+$ such that 
	\begin{equation}
		\sup_{\bm v \in \mathcal C}\left\|f_i(\bm v)\phi_i - f_i\left(\sum_{j=1}^k   \left(\bm v,\varphi_j\right)\varphi_j\right) \phi_i \right\|_{\mathcal Y} \le \frac{\epsilon}{3m}  \quad \forall i=1:m,
	\end{equation}
	where $\varphi_i$ are the $L^2$ orthogonal basis in Sobolev space $\mathcal X$.
	Then, for a specific $f_i: \mathcal X \mapsto \mathbb R$, let us define the following finite-dimensional continuous function $F_i: \mathbb R^k \mapsto \mathbb R$ as
	\begin{equation}
		F_i(x) = f_i\left(\sum_{j=1}^k x_j \varphi_j \right), \quad \forall x \in [-M,M]^{k},
	\end{equation}
where $M := \sup_{i}\sup_{v \in \mathcal C} (\bm v,\varphi_i) < \infty$ because of the compactness of $\mathcal C$.
	
	\paragraph{Universal approximation of $F_i$ on compact domain using Lemma~\ref{lem:NNAprox}:}
	For any $n\ge 2$, by using the scaling and shifting transformation to map $[-1,1]^k$ to $[0,1]^k$ and calling Lemma~\ref{lem:NNAprox} with $c_i = \|\phi_i\|_{\mathcal Y}=1$, we have 
	\begin{equation}\label{eq:tildeFiA}
		\widetilde F_i(x) = P_i \circ (g_i\circ)^{n} \circ L_i (x),
	\end{equation}
	for any $i=1:m$, such that
	\begin{equation}
		\left\| F_i(x) - \widetilde F_i (x)\right\|_{L^{\infty}([-M,M]^k)} \le \frac{\epsilon}{3m}.
	\end{equation}
	
	\paragraph{Construction of $\mathcal O_i$ using the structure of $\widetilde F_i$:} Now, let us define $\mathcal O_i$ as
	\begin{equation}\label{eq:mathcalOi}
		\mathcal O_i(\bm v) = \mathcal P_i \circ (\mathcal G_i\circ)^{n}\circ \mathcal L_i(\bm v) .
	\end{equation}
	More precisely, we can take
	\begin{equation}
		\mathcal L_i (\bm v) = L_i
		\begin{pmatrix}
			(\bm v, \varphi_1) \bm 1(x) \\ \vdots  \\ (\bm v, \varphi_k) \bm 1(x)
		\end{pmatrix}
	\end{equation}
where $L_i$ comes from~\eqref{eq:tildeFiA} when constructing $\widetilde F_i$ and $\bm 1(x) \in \mathcal Z$ denotes the constant function with value $1$ everywhere on $\Omega$.
	For $\mathcal G_i : \mathcal Z^{\widetilde k} \mapsto \mathcal Z^{\widetilde k} $, we denote the deep neural operator with ReLU activation function as
	\begin{equation}
		\mathcal G_i (\bm v) = \sigma\left(\mathcal W^L_{i}\sigma\left( \cdots  \mathcal W_i^2 \sigma(\mathcal W^1_i \bm v + \mathcal B^1_i) + \mathcal B_i^2\right) + \mathcal B_i^L \right).
	\end{equation}
	Furthermore, noticing the definition of $\mathcal L_i$, we take $\mathcal W_i^\ell$ and $\mathcal B_i^\ell$ as
	\begin{equation}
		\left[\mathcal W_i^\ell\right]_{st} (a\bm 1(x)) + \left[\mathcal B_{i}^\ell\right]_s = a \left[W_i^\ell\right]_{st} \bm 1(x) + \left[b_i^\ell\right]_s \bm 1(x),
	\end{equation}
	where $W_i^\ell$ and $b_i^\ell$ are defined by $g_i$ in \eqref{eq:tildeFiA} with
	\begin{equation}
		g_i(x) = \sigma\left( W^L_{i}\sigma\left( \cdots   W_i^2 \sigma( W^1_i x +  b^1_i) +  b_i^2\right) +  b_i^L \right).
	\end{equation}
	This leads to 
	\begin{equation}
		\begin{split}
			(\mathcal G_i\circ)^{n} \circ \mathcal L_i (\bm v) &= 
			\left(
			(\mathcal G_i\circ)^{n} \circ L_i
			\begin{pmatrix}
				(\bm v, \varphi_1)\bm 1(x)  \\ \vdots  \\ (\bm v, \varphi_k) \bm 1(x)
			\end{pmatrix}
			\right) \\
			&=
			\left(
			(g_i\circ)^{n} \circ L_i
			\begin{pmatrix}
				(\bm v, \varphi_1)  \\ \vdots  \\ (\bm v, \varphi_k) 
			\end{pmatrix}
			\right)\bm 1(x) \\
			&= \begin{pmatrix}
				c_{i,1}(\bm v) \bm 1(x) \\ \vdots \\c_{i,\widetilde k}(\bm v) \bm 1(x)
			\end{pmatrix}
		\end{split}
	\end{equation}
	where $c_{i,p}:\mathcal X \mapsto \mathbb R$ are some continuous functionals.
	
	Next, we can define 
	\begin{equation}
		\mathcal P_i 
		\begin{pmatrix}
			c_{i,1}(\bm v) \bm 1(x) \\ \vdots \\c_{i,\widetilde k}(\bm v) \bm 1(x)
		\end{pmatrix} = 
		\left(P_i \begin{pmatrix}
			c_{i,1}(\bm v) \bm 1(x) \\ \vdots\\ c_{i,\widetilde k}(\bm v)  \bm 1(x)
		\end{pmatrix}\right) \phi_i.
	\end{equation}
	
	In the end, we have
	\begin{equation}\label{eq:mathcalOiA}
		\begin{split}
			\mathcal O_i(\bm v) &=  \mathcal P_i \circ (\mathcal G_i\circ)^{n}\circ \mathcal L_i(\bm v) \\
			&= \left(
			P_i \circ (g_i\circ)^{n} \circ L_i 
			\begin{pmatrix}
				(\bm v, \varphi_1) \\ \vdots \\ (\bm v, \varphi_k)
			\end{pmatrix} \phi_i
			\right) \\
			&= \widetilde F_i \left(\left((\bm v, \varphi_1), \cdots, (\bm v, \varphi_k) \right)\right) \phi_i.
		\end{split}
	\end{equation}

	\paragraph{Triangle inequalities to finalize the proof:}
	Given the previous construction, we have
	\begin{equation}
		\begin{split}
			&\sup_{\bm v\in \mathcal C} \left\| f_i(\bm v) \phi_i - \mathcal O_i(\bm v)\right\|_{\mathcal Y} \\
			\le &\sup_{\bm v \in \mathcal C}\left\|f_i(\bm v)\phi_i - f_i\left(\sum_{j=1}^k   \left(\bm v,\varphi_j\right)\varphi_j\right) \phi_i \right\|_{\mathcal Y} \\
			+ &\sup_{\bm v \in \mathcal C} \left\|f_i\left(\sum_{j=1}^k   \left(\bm v,\varphi_j\right)\varphi_j\right) \phi_i - \mathcal O_i(\bm v)\right\|_{\mathcal Y} \\
			\le &\frac{\epsilon}{3m} + \sup_{\bm v \in \mathcal C} \left| f_i\left(\sum_{j=1}^k   \left(\bm v,\varphi_j\right)\varphi_j\right) - \widetilde F_i \left(\left((\bm v, \varphi_1), \cdots, (\bm v, \varphi_k) \right)\right)\right| \|\phi_i\|_{\mathcal Y} \\
			= & \frac{\epsilon}{3m} + \sup_{\bm v \in \mathcal C} \left| F_i \left(\left((\bm v, \varphi_1), \cdots, (\bm v, \varphi_k) \right)\right)- \widetilde F_i \left(\left((\bm v, \varphi_1), \cdots, (\bm v, \varphi_k) \right)\right)\right| \\
			\le &\frac{\epsilon}{3m} + \frac{\epsilon}{3m}  = \frac{2\epsilon}{3m}.
		\end{split}
	\end{equation}
	As a result, we have
	\begin{equation}
		\begin{split}
			&\sup_{\bm v\in \mathcal C} \left\| \mathcal O^*(\bm v) - \sum_{i=1}^m \mathcal O_i(\bm v) \right\|_{\mathcal Y} \\
			\le &\sup_{\bm v\in \mathcal C} \left( \left\| \mathcal O^*(\bm v) - \sum_{i=1}^m f_i(\bm v) \phi_i \right\|_{\mathcal Y} 
			+ \sum_{i=1}^m \|f_i(\bm v) \phi_i - \mathcal O_i(\bm v)\|_{\mathcal Y}  \right)\\
			\le &\frac{\epsilon}{3} + m \frac{2\epsilon}{3m} = \epsilon,
		\end{split}
	\end{equation}
	where $\mathcal O_i$ has the structure as in~\eqref{eq:mathcalOiA}.
	Since $\mathcal O_i$ has a uniform depth, we can then concatenate them to a global neural operator 
	\begin{equation}
		\mathcal O(\bm v) = \begin{pmatrix}
			\mathcal P_1 \\ \vdots\\ \mathcal P_m
		\end{pmatrix}    
		\circ 
		\left(\begin{pmatrix}
			\mathcal G_1 \\ \vdots \\\mathcal G_m 
		\end{pmatrix} \circ\right)^n
		\circ 
		\begin{pmatrix}
			\mathcal L_1 \\ \vdots \\\mathcal L_m
		\end{pmatrix} (\bm v).
	\end{equation}
	with the same depth but neurons as the summation for each sub-structure. 
\end{proof}

For $n=1$, this theorem degenerates to the most commonly used universal approximation theorem of classical neural operators, for example, \cite{kovachki2023neural,he2023mgno}.
For any $n\ge 2$, this result has not been mentioned or discussed in any previous literature. One of the most challenging parts in the above theorem is the self-composition, or one can imagine $(\mathcal G \circ)^n$ as a neural operator with $nL$ hidden layers where in which there is a sharing of parameters along each block which is a neural operator with $L$ hidden layers.

Motivated by the iterative methods in solving numerical PDEs, we have a natural question of whether we can have higher accuracy by just self-composing a fixed size (depth and width) backbone without adding any new parameters. Under a further assumption to the operator $\mathcal O^*$, we have the following theorem that with suitable and fixed width and depth of $\mathcal G$ and under a certain accuracy level, a larger $n$ can achieve a lower error level. 
\begin{theorem}\label{them:rate}
If $\mathcal O^*$ is Lipschitz continuous on a compact $\mathcal C\subset \mathcal X$ (with Lipschitz constant $\|\mathcal O^*\|$), then for any $\epsilon >0$, there exist fixed $N, \widetilde N, L \in \mathbb N^+$ (depending on $\epsilon$ but not on $n$), such that for any $n \in \mathbb N^+$ we can find bounded linear operators
$\mathcal L: \mathcal X \to \mathcal Z^{\widetilde N}$ and $\mathcal P: \mathcal Z^{\widetilde N} \to \mathcal Y$ and a ReLU neural operator $\mathcal G: \mathcal Z^{\widetilde N} \to \mathcal Z^{\widetilde N}$ with $L$ hidden layers and $N$ neurons in each layer satisfying
    \begin{equation}
    \sup_{\bm v \in \mathcal C}\left\| \mathcal O^*(\bm v) - \mathcal P \circ (\mathcal G \circ)^n \circ \mathcal L (\bm v)\right\|_{\mathcal X} \le \epsilon + \frac{C_\epsilon}{\log(n)},
\end{equation}
where $C_\epsilon$ depends on $\epsilon$, the Lipschitz constant of $\mathcal O^*$, and finite-dimensional projection choices, but is independent of $n$. The quantities $L, N, \widetilde N$, and the projection dimension are chosen as functions of $\epsilon$ and remain fixed as $n$ grows.
Here, $\mathcal Z$ can be any function spaces that contain the constant function $\bm 1(x)$ on $\Omega$.
\end{theorem}
The proof mainly follows the steps in the proof of Theorem~\ref{them:approxi}. The key here is to use the Lipschitz condition of the target operator to get an explicit approximation rate in the final bound.
\begin{remark}
While the bound decays as $O(1/\log n)$, it is useful in the fixed-parameter regime: after selecting $L,N,\widetilde N$, increasing $n$ improves accuracy without increasing parameter count. Empirically, we often observe faster decay (Fig.~\ref{fig:conv_curves}b), but the theorem guarantees monotone improvement with depth.
\end{remark}

\begin{proof}
	We split the proof as follows.
	\paragraph{Finite-dimension approximation of $\mathcal O^*$ by projection with Lipschitz continuity} 
	As in the first step in the proof of Theorem~\ref{them:approxi}, those continuous functionals $f_i: \mathcal X \mapsto \mathbb R$ with orthogonal basis $\{\phi_1, \cdots, \phi_m \} \subset \mathcal Y$ for $i=1:m$ are actually defined as 
	$$
	f_i(\bm v) := (\phi_i, \mathcal O^*(\bm v))_{\mathcal Y},
	$$ which are also Lipschitz continuous and
	\begin{equation}
		\sup_{\bm v\in \mathcal C} \left\| \mathcal O^*(\bm v) - \sum_{i=1}^m f_i(\bm v) \phi_i \right\|_{\mathcal Y} \le \frac{\epsilon}{2}.
	\end{equation}
	Moreover, it is easy to verify that the Lipschitz constants of $f_i$ are all bounded by the Lipschitz constant of $\mathcal O^*$, which we denote as $\|\mathcal O^*\|$.
	Thus, we only need to prove that there is deep neural operator $\mathcal O_i = \mathcal P_i \circ (\mathcal G_i\circ)^{n} \circ \mathcal L_i$ such that
	\begin{equation}
		\sup_{\bm v\in \mathcal C} \left\| f_i(\bm v) \phi_i - \mathcal O_i(\bm v)\right\|_{\mathcal Y} \le \frac{\epsilon}{2m} + \frac{C_\epsilon}{m\log(n)},
	\end{equation}
	where $C_\epsilon$ depends only on $\epsilon$ and $\mathcal O^*$ and does not depends on $n$.
	
	\paragraph{Parameterization (approximation) of $\mathcal X$ with finite dimensions to discretize $f_i$}
	Since $\mathcal X = H^{s}(\Omega)$ and $\mathcal C$ is compact, we can find $k \in \mathbb N^+$ such that 
	\begin{equation}
		\sup_{\bm v \in \mathcal C}\left\|f_i(\bm v)\phi_i - f_i\left(\sum_{j=1}^k   \left(\bm v,\varphi_j\right)\varphi_j\right) \phi_i \right\|_{\mathcal X} \le \frac{\epsilon}{2m}  \quad \forall i=1:m,
	\end{equation}
	where $\varphi_i$ are the orthogonal basis in $H^{s}(\Omega)$.
	Then, for a specific $f_i: \mathcal X \mapsto \mathbb R$, let us define the following finite-dimensional continuous function $F_i: \mathbb R^k \mapsto \mathbb R$ as
	\begin{equation}
		F_i(x) = f_i\left(\sum_{j=1}^k x_j \varphi_j \right), \quad \forall x \in [-M,M]^{k},
	\end{equation}
	where $M := \sup_{i}\sup_{u \in \mathcal C} (\bm u,\varphi_i) < \infty$ because of the compactness of $\mathcal C$. Furthermore, for each $i=1:m$, we notice that $F_i(x)$ is also Lipschitz continuous with Lipschitz constant less than the Lipschitz constant of $f_i$ which has a uniform bound $\|\mathcal O^*\|$.
	
	\paragraph{Quantitative approximation of $F_i$ using self-comopsing deep ReLU neural networks} For any $n$, by using the scaling and shifting transformation to map $[-M,M]^k$ to $[-1,1]^k$ and calling Theorem 1.3 in~\cite{zhang2023enhancing}, we have
	\begin{equation}\label{eq:tildeFiB}
		\widetilde F_i(x) = P_i \circ (g_i\circ)^{n} \circ L_i (x),
	\end{equation}
	for any $i=1:m$, such that
	\begin{equation}
		\left\| F_i(x) - \widetilde F_i (x)\right\|_{L^{\infty}([-M,M]^k)} \le \frac{12\sqrt{k}\|\mathcal O^*\|_{\mathcal Y}}{ n^{1/k}} \le \frac{C_\epsilon}{m\log(n)}
	\end{equation}
	for sufficient large $n$ and $C_\epsilon = \mathcal O\left( \sqrt{k}m\|\mathcal O^*\|\right)$ does not depend on $n$.
	Here, $L_i: \mathbb R^k \mapsto \mathbb R^{\widetilde k}$ and $P_i: \mathbb R^{\widetilde k} \mapsto \mathbb R$ are affine mappings and $g_i: \mathbb R^{\widetilde k} \mapsto\mathbb R^{\widetilde k} $ is a deep ReLU neural network function with $L$ hidden layers and $\overline k_i$ neurons in each layer.
	
	The construction of $\mathcal{O}_i$ based on the structure of $\widetilde{F}_i$, as well as the verification of the final results, closely follows the approach used in the proof of Theorem~\ref{them:approxi}. Nevertheless, we include the details here for completeness.
	\paragraph{Construction of $\mathcal O_i$ using the structure of $\widetilde F_i$:} Now, let us define $\mathcal O_i$ as
	\begin{equation}
		\mathcal O_i(\bm v) = \mathcal P_i \circ (\mathcal G_i\circ)^{n}\circ \mathcal L_i(\bm v) .
	\end{equation}
	More precisely, we can take
	\begin{equation}
		\mathcal L_i (\bm v) = L_i
		\begin{pmatrix}
			(\bm v, \varphi_1) \bm 1(x) \\ \vdots  \\ (\bm v, \varphi_k) \bm 1(x)
		\end{pmatrix}
	\end{equation}
where $L_i$ comes from \eqref{eq:tildeFiB} when constructing $\widetilde F_i$ and $\bm 1(x) \in \mathcal Z$ denotes the constant function with value $1$.
	For $\mathcal G_i : \mathcal Z^{\widetilde k} \mapsto \mathcal Z^{\widetilde k} $, we denote the deep neural operator with ReLU activation function as
	\begin{equation}
		\mathcal G_i (\bm v) = \sigma\left(\mathcal W^L_{i}\sigma\left( \cdots  \mathcal W_i^2 \sigma(\mathcal W^1_i \bm v + \mathcal B^1_i) + \mathcal B_i^2\right) + \mathcal B_i^L \right).
	\end{equation}
	Furthermore, noticing the definition of $\mathcal L_i$, we take $\mathcal W_i^\ell$ and $\mathcal B_i^\ell$ as
	\begin{equation}
		\left[\mathcal W_i^\ell\right]_{st} (a\bm 1(x)) + \left[\mathcal B_{i}^\ell\right]_s = a \left[W_i^\ell\right]_{st} \bm 1(x) + \left[b_i^\ell\right]_s \bm 1(x),
	\end{equation}
	where $W_i^\ell$ and $b_i^\ell$ are defined by $g_i$ in \eqref{eq:tildeFiB} with
	\begin{equation}
		g_i(x) = \sigma\left( W^L_{i}\sigma\left( \cdots   W_i^2 \sigma( W^1_i x +  b^1_i) +  b_i^2\right) +  b_i^L \right).
	\end{equation}
	This leads to 
	\begin{equation}
		\begin{split}
			(\mathcal G_i\circ)^{n} \circ \mathcal L_i (\bm v) &= 
			\left(
			(\mathcal G_i\circ)^{n} \circ L_i
			\begin{pmatrix}
				(\bm v, \varphi_1)\bm 1(x)  \\ \vdots  \\ (\bm v, \varphi_k) \bm 1(x)
			\end{pmatrix}
			\right) \\
			&=
			\left(
			(g_i\circ)^{n} \circ L_i
			\begin{pmatrix}
				(\bm v, \varphi_1)  \\ \vdots  \\ (\bm v, \varphi_k) 
			\end{pmatrix}
			\right)\bm 1(x) \\
			&= \begin{pmatrix}
				c_{i,1}(\bm v) \bm 1(x) \\ \vdots \\c_{i,\widetilde k}(\bm v) \bm 1(x)
			\end{pmatrix}
		\end{split}
	\end{equation}
	where $c_{i,p}:\mathcal X \mapsto \mathbb R$ are some continuous functionals.
	
	Next, we can define 
	\begin{equation}
		\mathcal P_i 
		\begin{pmatrix}
			c_{i,1}(\bm v) \bm 1(x) \\ \vdots \\c_{i,\widetilde k}(\bm v) \bm 1(x)
		\end{pmatrix} = 
		\left(P_i \begin{pmatrix}
			c_{i,1} \bm 1(x) \\ \vdots\\ c_{i,\widetilde k}  \bm 1(x)
		\end{pmatrix}\right) \phi_i.
	\end{equation}
	
	In the end, we have
	\begin{equation}\label{eq:mathcalOi1}
		\begin{split}
			\mathcal O_i(\bm v) &=  \mathcal P_i \circ (\mathcal G_i\circ)^{n}\circ \mathcal L_i(\bm v) \\
			&= \left(
			P_i \circ (g_i\circ)^{n} \circ L_i 
			\begin{pmatrix}
				(\bm v, \varphi_1) \\ \vdots \\ (\bm v, \varphi_k)
			\end{pmatrix} \phi_i
			\right) \\
			&= \widetilde F_i \left(\left((\bm v, \varphi_1), \cdots, (\bm v, \varphi_k) \right)\right) \phi_i.
		\end{split}
	\end{equation}

	\paragraph{Triangle inequalities to finalize the proof}
	Given the previous construction, we have
	\begin{equation}
		\begin{split}
			&\sup_{\bm v\in \mathcal C} \left\| f_i(\bm v) \phi_i - \mathcal O_i(\bm v)\right\|_{\mathcal Y} \\
			\le &\sup_{\bm v \in \mathcal C}\left\|f_i(\bm v)\phi_i - f_i\left(\sum_{j=1}^k   \left(\bm v,\varphi_j\right)\varphi_j\right) \phi_i \right\|_{\mathcal Y} \\
			+ &\sup_{\bm v \in \mathcal C} \left\|f_i\left(\sum_{j=1}^k   \left(\bm v,\varphi_j\right)\varphi_j\right) \phi_i - \mathcal O_i(\bm v)\right\|_{\mathcal Y} \\
			\le &\frac{\epsilon}{2m} + \sup_{\bm v \in \mathcal C} \left| f_i\left(\sum_{j=1}^k   \left(\bm v,\varphi_j\right)\varphi_j\right) - \widetilde F_i \left(\left((\bm v, \varphi_1), \cdots, (\bm v, \varphi_k) \right)\right)\right| \|\phi_i\|_{\mathcal Y} \\
			= & \frac{\epsilon}{2m} + \sup_{\bm v \in \mathcal C} \left| F_i \left(\left((\bm v, \varphi_1), \cdots, (\bm v, \varphi_k) \right)\right)- \widetilde F_i \left(\left((\bm v, \varphi_1), \cdots, (\bm v, \varphi_k) \right)\right)\right| \\
			\le &\frac{\epsilon}{2m} + \frac{C_\epsilon}{m \log(n)}.
		\end{split}
	\end{equation}
	As a result, we have
	\begin{equation}
		\begin{split}
			&\sup_{\bm v\in \mathcal C} \left\| \mathcal O^*(\bm v) - \sum_{i=1}^m \mathcal O_i(\bm v) \right\|_{\mathcal Y} \\
			\le &\sup_{\bm v\in \mathcal C} \left( \left\| \mathcal O^*(\bm v) - \sum_{i=1}^m f_i(\bm v) \phi_i \right\|_{\mathcal Y} 
			+ \sum_{i=1}^m \|f_i(\bm v) \phi_i - \mathcal O_i(\bm v)\|_{\mathcal Y}  \right)\\
			\le &\frac{\epsilon}{3} + m \left(\frac{\epsilon}{2m} + \frac{C_\epsilon}{m \log(n)}\right) = \epsilon + \frac{C_\epsilon}{\log(n)},
		\end{split}
	\end{equation}
	where $\mathcal O_i$ has the structure as in~\eqref{eq:mathcalOi1}.
	Since $\mathcal O_i$ has a uniform depth, we can then concatenate them to a global neural operator 
	\begin{equation}
		\mathcal O(\bm v) = \begin{pmatrix}
			\mathcal P_1 \\ \vdots\\ \mathcal P_m
		\end{pmatrix}    
		\circ 
		\left(\begin{pmatrix}
			\mathcal G_1 \\ \vdots \\\mathcal G_m 
		\end{pmatrix} \circ\right)^n
		\circ 
		\begin{pmatrix}
			\mathcal L_1 \\ \vdots \\\mathcal L_m
		\end{pmatrix} (\bm v).
	\end{equation}
	with the same depth but neurons as the summation for each sub-structure. 
	
\end{proof}

Noticing that $N, \widetilde N$, and $L$ are independent of $n$, the above theorem indeed verifies our previous intuition that with a certain budget for the backbone neural operator and under a certain accuracy level, the deeper (larger $n$), the better (lower accuracy). 

\paragraph{Discussion on the convergence rate:} Theorem~\ref{them:rate} provides a quantitative guarantee that the approximation error decreases as the number of compositions $n$ increases, specifically at a rate of $O(1/\log(n))$ beyond a base error $\epsilon$. While this confirms the benefit of deeper composition, the $1/\log(n)$ rate is asymptotically slow compared to error bounds for standard deep neural networks, where error often decreases polynomially or exponentially with explicit increases in width or depth (number of unique layers/parameters). This theoretical rate might suggest that the primary practical advantages of the self-composing structure stem from significant parameter sharing (acting as regularization) and the effectiveness of the Train-and-Unroll strategy (providing good initialization and progressive learning), rather than purely from the asymptotic error reduction due to $n$. The observed empirical error decay (e.g., Figure~\ref{fig:conv_curves}(b)) often appears faster than $1/\log(n)$, suggesting this bound might not be tight or that the constants involved play a significant role in practical regimes. Further investigation is needed to bridge the gap between this theoretical rate and empirical observations, potentially requiring different proof techniques or assumptions.

Numerically, the depth and accuracy scaling in Figure~\ref{fig:conv_curves}(b) verifies this theoretical result, showing error reduction with increasing $n$. Furthermore, we point out that the empirical scaling often appears closer to linear or sub-linear in the log-error plot, which is faster than the $O(1/\log(n))$ rate suggested by the theorem. This discrepancy highlights an interesting theoretical question for future investigation.

\subsection{Exponential accuracy scaling for Darcy and Helmholtz problems}
\label{sec:exp-scaling}
While Theorem~\ref{them:rate} provides a general-purpose approximation guarantee with a logarithmic rate, a much stronger result can be established by leveraging the specific structure of the self-composing operator. When the backbone $\mathcal{G}_\theta$ learns to approximate a contractive iterative update, a property inherent to many effective numerical solvers, the error decays exponentially with depth. This section formalizes this connection, bridging the gap between the proposed architecture and classical fixed-point theory (cf.~\cite{feischl2025neural}). These results apply under standard assumptions for elliptic (Darcy) and indefinite (Helmholtz) operators, providing a rigorous foundation for the high efficiency observed in our experiments.

Throughout, let $\bm v$ denote the PDE coefficients and forcing, and let $u^*(\bm v)$ be the unique weak solution of the PDE posed on a bounded Lipschitz domain.

\paragraph{A contractive fixed-point view.}
Suppose that for each $\bm v$ there exist bounded linear operators $A_{\bm v}:V\to V'$ and $B_{\bm v}:V'\to V$ and a load $F_{\bm v}\in V'$ such that $u^*(\bm v)$ solves $A_{\bm v}u=F_{\bm v}$ and the one-step update
\begin{equation}\label{eq:ideal-update}
\mathcal T_{\bm v}(u) := u + B_{\bm v}\big(F_{\bm v}-A_{\bm v}u\big)
\end{equation}
is a contraction in the energy norm $\|w\|_{B_{\bm v}^{-1}} := (B_{\bm v}^{-1}w,w)^{1/2}$ with factor $\rho\in(0,1)$ uniformly over $\bm v$, i.e.,
$\|\mathcal T_{\bm v}(u)-\mathcal T_{\bm v}(w)\|_{B_{\bm v}^{-1}}\le \rho\,\|u-w\|_{B_{\bm v}^{-1}}$.
Assume our learned backbone $\mathcal G_\theta$ realizes a perturbed update $\widetilde{\mathcal T}_{\bm v}$ that approximates \eqref{eq:ideal-update} uniformly,
\begin{equation}\label{eq:block-bias}
\sup_{u,\bm v}\,\big\|\widetilde{\mathcal T}_{\bm v}(u)-\mathcal T_{\bm v}(u)\big\|_{B_{\bm v}^{-1}}\le \varepsilon_b.
\end{equation}
Then the $n$-fold self-composed operator with initialization $u^0$ satisfies a geometric accuracy law.

\begin{theorem}[Geometric decay for contractive updates]\label{thm:contract}
Under the assumptions above, for all $n\ge1$,
\begin{equation}\label{eq:geom-bound-general}
\big\|u^*(\bm v)-\big(\widetilde{\mathcal T}_{\bm v}\circ\big)^n(u^0)\big\|_{B_{\bm v}^{-1}}
\;\le\; \rho^{n}\,\big\|u^*(\bm v)-u^0\big\|_{B_{\bm v}^{-1}}\;+\;
\frac{1-\rho^{n}}{1-\rho}\,\varepsilon_b.
\end{equation}
Consequently, the error decays exponentially in depth up to a bias floor proportional to the per-block approximation error $\varepsilon_b$.
\end{theorem}
\begin{proof}
Because $\mathcal T_{\bm v}$ is a contraction with fixed point $u^*(\bm v)$ and $\widetilde{\mathcal T}_{\bm v}$ satisfies \eqref{eq:block-bias}, a standard perturbation of Banach's fixed-point iteration yields
$\|u^*-\widetilde{u}^{k+1}\|_{B^{-1}}\le \rho\,\|u^*-\widetilde{u}^{k}\|_{B^{-1}}+\varepsilon_b$ with $\widetilde{u}^{k+1}=\widetilde{\mathcal T}_{\bm v}(\widetilde{u}^{k})$.
Unrolling this recursion gives \eqref{eq:geom-bound-general}.
\end{proof}

The bias term in \eqref{eq:geom-bound-general} is crucial for interpretation: unless the per-block approximation error $\varepsilon_b$ is itself driven close to zero, the iterates saturate at an error floor of order $\varepsilon_b/(1-\rho)$. Therefore deeper self-composition alone does not imply machine-precision accuracy, even though it yields an initial exponential error decay.

\paragraph{Elliptic Darcy operators.}
Consider the Darcy problem $-\nabla\cdot(a(x)\nabla u)=f$ with $ a_{\min} \leq a(x)  \leq a_{\max}$ and homogeneous Dirichlet boundary conditions. Let $A_{\bm v}$ be the (SPD) stiffness operator associated with the bilinear form $a(u,v)=\int_\Omega a(x)\nabla u\cdot\nabla v\,dx$, and let $B_{\bm v}$ be a multigrid $V$-cycle preconditioner. Under standard assumptions, $B_{\bm v}$ is spectrally equivalent to $A_{\bm v}^{-1}$ uniformly in the mesh size (though the constants may depend on the contrast $a_{\max}/a_{\min}$). This implies that the error-propagation operator $I-B_{\bm v}A_{\bm v}$ is a contraction in the $\|\cdot\|_{B_{\bm v}^{-1}}$-norm with a factor $\rho\in(0,1)$ independent of the discretization \cite{hackbusch2013multi,trottenberg2000multigrid,xu1989theory}. Therefore Theorem~\ref{thm:contract} applies.
\begin{corollary}[Mesh-independent geometric rate for Darcy]\label{cor:darcy}
If $\mathcal G_\theta$ approximates one $V$-cycle update so that \eqref{eq:block-bias} holds with $\varepsilon_b$, then for Darcy flow with coefficients in a fixed range $[a_{\min}, a_{\max}]$,
\begin{equation}
\big\|u^*(\bm v)-\big(\widetilde{\mathcal T}_{\bm v}\circ\big)^n(u^0)\big\|_{B_{\bm v}^{-1}}\;\le\; \rho^{n}\,\big\|u^*(\bm v)-u^0\big\|_{B_{\bm v}^{-1}}+\frac{1-\rho^{n}}{1-\rho}\,\varepsilon_b,
\end{equation}
with $\rho\in(0,1)$ independent of mesh size.
\end{corollary}

\paragraph{Helmholtz with impedance boundary conditions.}

For the indefinite Helmholtz operator $A_k = -\Delta - k^2$, establishing a rigorous contractive property for simple stationary iterations is significantly more challenging than in the elliptic case. Standard preconditioners, such as the shifted Laplacian $M_{k,\sigma}^{-1} \approx (-\Delta - (1-i\sigma)k^2)^{-1}$, are typically used within Krylov subspace methods (e.g., GMRES) rather than as fixed-point iterations, because the spectrum of the preconditioned operator $M_{k,\sigma}^{-1} A_k$ often lies in a cluster that does not guarantee $\rho(I - M_{k,\sigma}^{-1} A_k) < 1$ for high wavenumbers $k$.

However, the self-composing neural operator framework is not restricted to mimicking simple linear stationary iterations. The backbone $\mathcal{G}_\theta$ has the capacity to learn a non-linear, optimized update rule that may behave more like a locally contractive map or a step of a more sophisticated solver. In our USCT experiments (Section~\ref{sec:numerics}), we observe that the error decays geometrically with depth (see Figure~\ref{fig:conv_curves}(b)), strongly suggesting that the trained model has successfully learned a contractive update strategy for the distribution of wave numbers in the dataset. This empirical behavior aligns with the geometric convergence predicted by Theorem~\ref{thm:contract}, even if a formal proof of contractivity for a standard numerical baseline remains elusive in the high-frequency regime. This highlights the potential of learning-based methods to discover effective iterative solvers where classical theory is limited.

\subsection{Experiments for self-composing neural operators}

\begin{table}[ht]
\caption{Performance comparison for Darcy benchmarks. Errors reported as relative $L^2$ ($\times 10^{-2}$) and relative $H^1$ ($\times 10^{-2}$). $H^1$ = relative $H^1$ semi-norm error, GT = Galerkin Transformer~\cite{cao2021choose}, LSM = Latent Spectral Model~\cite{wu2023solving}. The cost columns are profiled on the standard $421\times421$ Darcy Rough grid on one NVIDIA RTX A100 GPU.}
\label{tab:darcy}
	\begin{center}
	\footnotesize
	\resizebox{\textwidth}{!}{%
	\begin{tabular}{lccccccccccc}
	\toprule
	\textbf{Model} & \textbf{Time} & \textbf{Total Time} & \textbf{Mem} & \textbf{Params} & & \multicolumn{2}{c}{\textbf{Darcy smooth}} & \multicolumn{2}{c}{\textbf{Darcy rough}} & \multicolumn{2}{c}{\textbf{Multiscale}}\\
	& \textbf{(s/iter)} & \textbf{(h)} & \textbf{(GB)} & \textbf{(M)} & & $L^2$ & $H^1$ & $L^2$ & $H^1$ & $L^2$ & $H^1$ \\ 
	\midrule
	\multicolumn{12}{l}{\textit{Baseline methods}} \\
	\midrule
	GT & 38.2 & 5.3 & 3.5 & 2.22 & & 0.945 & 3.365 & 1.790 & 6.269 & 1.052 & 8.207 \\
	LSM & 18.2 & 2.5 & 2.1 & 4.81 & & 0.601 & 2.610 & 2.658 & 4.446 & 1.050 & 4.226 \\
	\midrule
	\multicolumn{12}{l}{\textit{Neural operators without self-composition}} \\
	\midrule
	FNO2D & 7.4 & 1.0 & 2.7 & 2.37 & & 0.684 & 2.583 & 1.613 & 7.516 & 1.800 & 9.619 \\
	MgNO & \textbf{6.6} & 0.9 & 6.4 & \textbf{0.57} & & \textbf{0.153} & \textbf{0.711} & \textbf{0.339} & \textbf{1.380} & \textbf{0.715} & \textbf{1.756} \\
	\midrule
	\multicolumn{12}{l}{\textit{Neural operators with self-composition }} \\
	\midrule
	FNO2D-self & 7.4 & 1.0 & 2.7 & 0.56 & & 0.751 & 3.102 & 1.981 & 8.115 & 2.064 & 11.070 \\
	MgNO-self & \textbf{6.7} & 0.9 & 6.4 & \textbf{0.17} & & \textbf{0.187} & \textbf{0.813} & \textbf{0.371} & \textbf{1.514} & \textbf{0.800} & \textbf{2.871} \\
	\bottomrule
	\end{tabular}
	}
	\end{center}
	\vspace{2mm}
	\end{table}
The Darcy cost columns in Table~\ref{tab:darcy} should be interpreted as end-to-end surrogate training costs on the $421\times421$ Darcy Rough grid; on the same RTX A100 GPU, a single discrete residual evaluation takes about $0.70$ ms for one sample and $0.80$ ms for a batch of 10.
	\begin{figure}[H]
	\centering
	\subfloat{\includegraphics[width=0.23\textwidth]{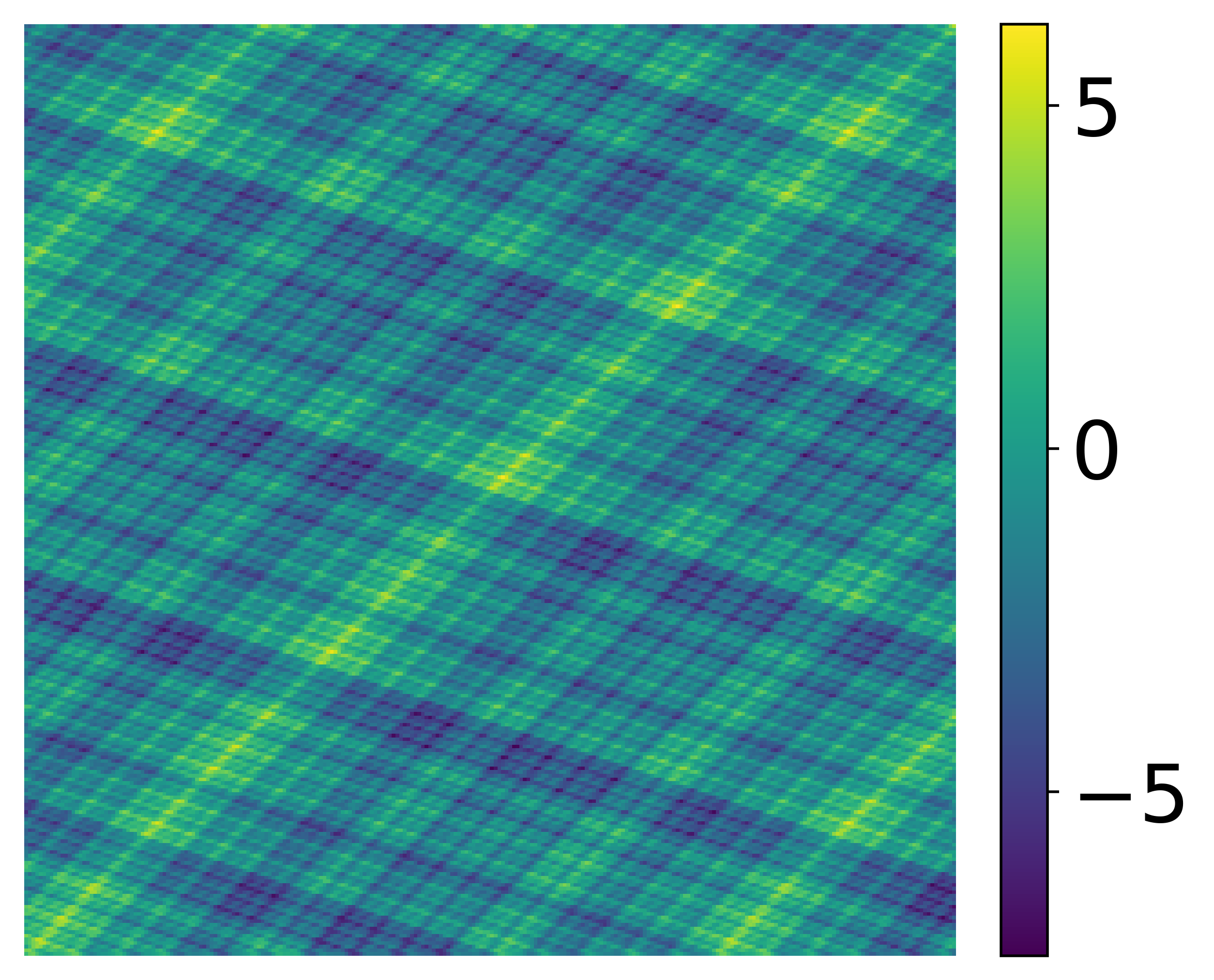}}
	\subfloat{\includegraphics[width=0.22\textwidth]{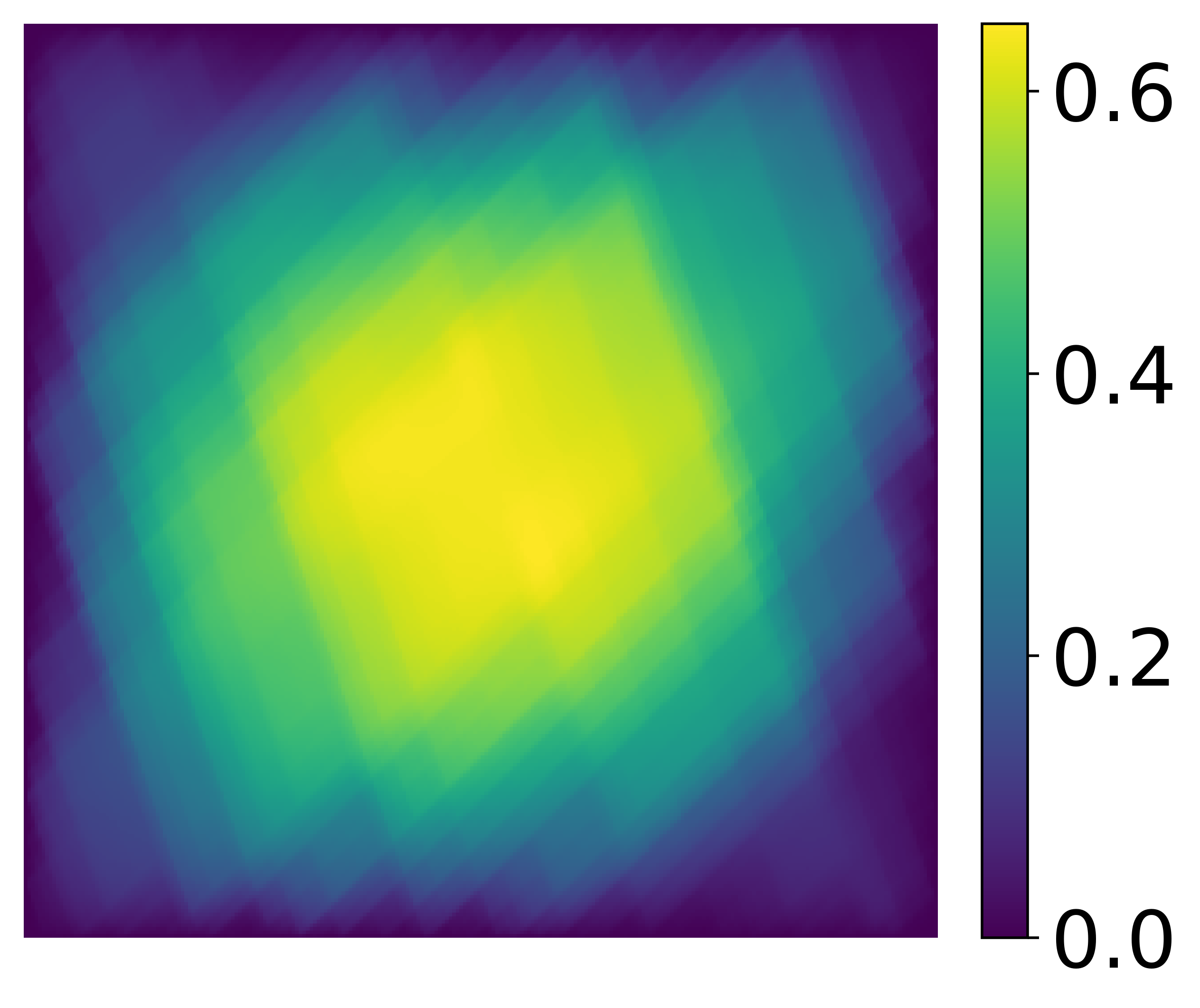}}
	\subfloat{\includegraphics[width=0.22\textwidth]{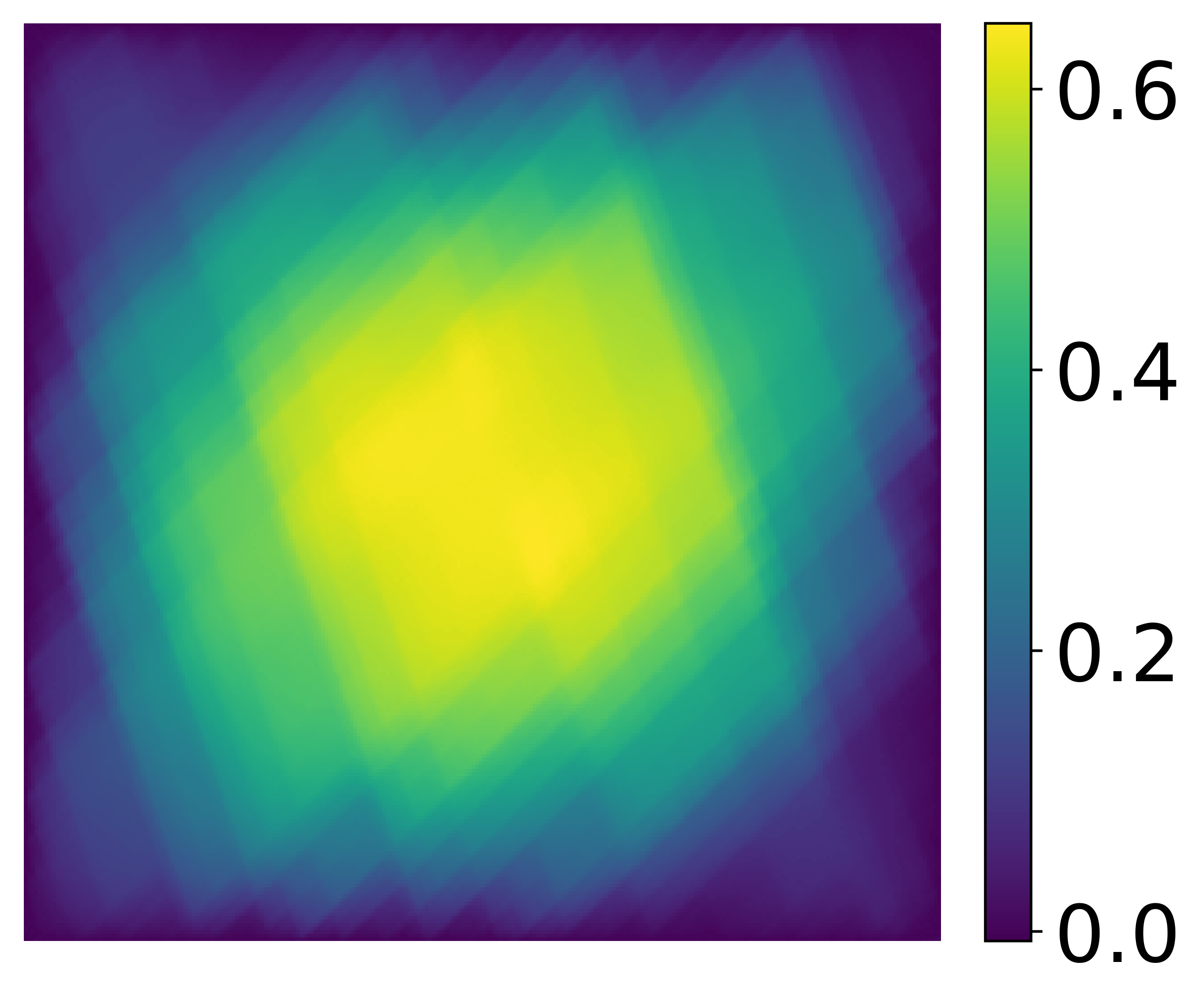}}
	\subfloat{\includegraphics[width=0.25\textwidth]{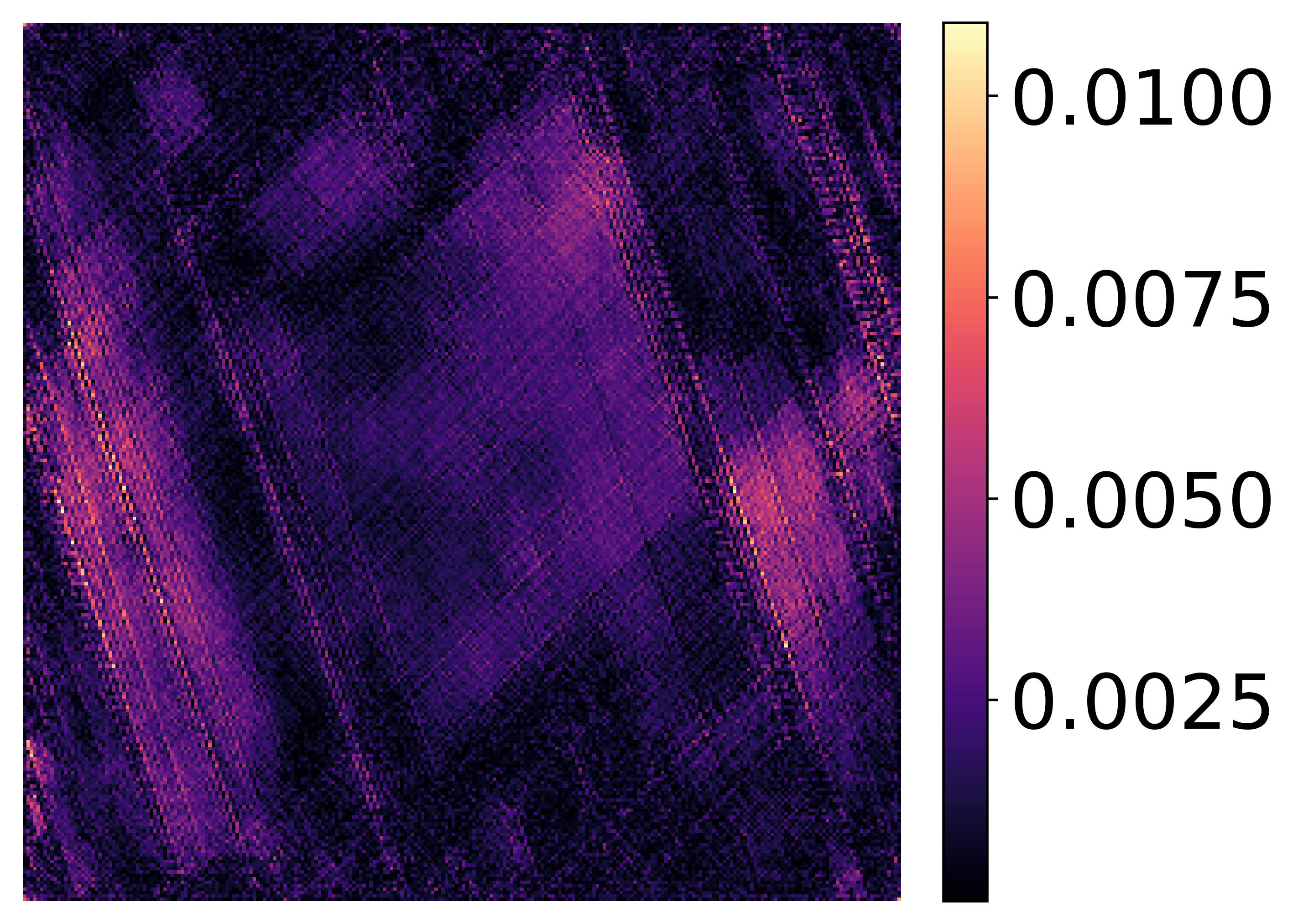}}
	\caption{Multiscale Darcy flow test case. (a) Input coefficient field (log2 scale). (b) Ground-truth solution. (c) MgNO-self prediction. (d) Absolute error $|u_{\mathrm{pred}}-u_{\mathrm{true}}|$. }
	\label{fig:example-4panel}
\end{figure}

Table~\ref{tab:darcy} offers initial experimental evidence for the practical benefits of the self-composing neural operator architecture. The results on Darcy flow benchmarks \cite{li2020fourier} and multiscale benchmark \cite{liu2022ht} demonstrate that models employing shared-parameter self-composition (denoted with ``-self", e.g., FNO2D-self, MgNO-self) attain competitive approximation accuracy in both relative $L^2$ and $H^1$ errors when compared to their counterparts without parameter sharing (e.g., FNO2D~\cite{li2020fourier}, MgNO~\cite{he2023mgno}). The performance on the multiscale benchmark is particularly noteworthy. For the Darcy experiments, MgNO follows the standard six-block MgNO configuration with width 32 and three input channels on the $421\times421$ benchmark grid; each block uses $3\times3$ convolutional operator and smoother maps together with strided restriction and transposed-convolution prolongation across the multigrid hierarchy. Table~\ref{tab:darcy} highlights the parameter efficiency of self-composed models. MgNO-self achieves competitive accuracy with only 0.17M parameters compared to 0.57M for MgNO, while FNO2D-self uses 0.56M versus 2.37M for FNO2D. This substantial reduction in size confirms that iteratively applying a shared backbone $\mathcal{G}_\theta$ yields deep, expressive operators with far fewer trainable degrees of freedom.  It is worth noting, however, that lower parameter count does not by itself imply lower activation memory or shorter wall-clock training time. For matched unroll depth, the repeated forward/backward graph is still executed sequentially, so peak training memory remains governed largely by activations rather than by the smaller tied parameter set. The main practical benefit is therefore parameter efficiency and regularization and the possibility for a more efficient training process as will be discussed in Section~\ref{sec:train}. Further evaluations on the Helmholtz equation and the Train-and-Unroll strategy follow in Section~\ref{sec:numerics}.

Figure~\ref{fig:example-4panel} illustrates a representative multiscale Darcy prediction and its pointwise error.

\section{Dynamic depth training algorithm}\label{sec:train}
The complete Train-and-Unroll procedure is summarized in Algorithm~\ref{alg:train-unroll}.
In this section, we propose a novel training strategy for deep self-composed neural operators, termed the \textbf{Train-and-Unroll} approach. Instead of training a deep model with fixed depth directly, \textbf{Train-and-Unroll} incrementally increases the number of self-compositions during training, starting from a shallow architecture. 

\paragraph{Motivation}
Training deep neural operators $\mathcal{O}^{(\ell)} = \mathcal{P} \circ (\mathcal{G}_\theta \circ)^\ell \circ \mathcal{L}$ is challenging, as determining the optimal depth $n$ is difficult and direct training is computationally expensive. Balancing complexity and accuracy typically requires resource-intensive experimentation.

Our \textbf{Train-and-Unroll} strategy mitigates this by incrementally increasing depth $\ell$ during training. We initialize each deeper stage with weights from the converged shallower model. This approach leverages the shared backbone $\mathcal{G}_\theta$ to accelerate convergence and allows for adaptive depth selection based on performance and budget.

\begin{algorithm}[H] \label{alg:train-unroll}
\caption{Train-and-Unroll Algorithm (Inputs/Outputs and complexity). Inputs: training set $\mathcal D$, maximum depth $n$, backbone $\mathcal G_\theta$, projection $\mathcal L$, and readout $\mathcal P$. Outputs: trained models $\mathcal{O}^{(1)},\ldots,\mathcal{O}^{(n)}$. Per-stage complexity is dominated by applying $(\mathcal G_\theta)^\ell$ and backpropagation; the parameter count remains fixed while activation memory scales with $\ell$. The loss function  at stage $\ell$ is defined as 
$\mathrm{Loss} (\theta; \ell) = \dfrac{1}{N} \sum_{j=1}^N \left\| \mathcal{O}^{(\ell)}(\bm{u}_j) - \bm{v}_j \right\|^2$.
}
\begin{algorithmic}[1]
\STATE \textbf{Initialization:} Initialize $\mathcal{G}_\theta$, $\mathcal{P}$, $\mathcal{L}$. Set initial model $\mathcal{O}^{(1)} := \mathcal{P} \circ (\mathcal{G}_\theta \circ)^1 \circ \mathcal{L}.$
\STATE \textbf{Step 1:} Train $\mathcal{O}^{(1)}$ by minimizing $\mathrm{Loss}(\theta; 1)$ until convergence, yielding $\theta^{(1)}$.
\STATE \textbf{Step 2:} Incrementally increase depth:
\FOR{$\ell = 2$ to $n$}
\STATE Initialize $\mathcal{O}^{(\ell)} := \mathcal{P} \circ (\mathcal{G}_{\theta^{(\ell-1)}} \circ)^\ell \circ \mathcal{L}$ using parameters $\theta^{(\ell-1)}$.
\STATE Train $\mathcal{O}^{(\ell)}$ by minimizing $\mathrm{Loss}(\theta; \ell)$ until convergence, obtaining $\theta^{(\ell)}$.
\ENDFOR
\STATE \textbf{Output:} A series of trained models $\mathcal{O}^{(1)}, \dots, \mathcal{O}^{(n)}$ with parameters $\theta^{(1)}, \dots, \theta^{(n)}$.
\end{algorithmic}
\end{algorithm}

\begin{remark}[Parameter count and gradient path]
For models of the form $\mathcal{O}^{(\ell)}=\mathcal P\circ(\mathcal G_\theta\circ)^{\ell}\circ\mathcal L$ with tied $\theta$, the trainable parameter count is independent of $\ell$. Increasing $\ell$ increases the backpropagation path length linearly in $\ell$, potentially affecting gradient stability; the Train-and-Unroll curriculum mitigates this by warm-starting deeper stages from shallower optima (cf. Fig.~\ref{fig:conv_curves}a).
\end{remark}

\begin{figure}[H]
\centering

\begin{tikzpicture}[
    node distance=0.7cm and 0.7cm, 
    stage/.style={
        rectangle, draw, rounded corners, thick,
        fill=blue!10,
        minimum height=3cm, 
        minimum width=2.8cm, 
        text centered, text width=2.6cm, 
        drop shadow,
        font=\footnotesize 
    },
    param_transfer/.style={
        font=\scriptsize, text centered, text width=1.5cm, inner sep=1pt 
    },
    arrow/.style={-Stealth, thick, draw=blue!60!black},
    titlefont/.style={font=\bfseries\small} 
]

\node[rectangle, draw, dashed, fill=yellow!15, thick, rounded corners,
      text width=3.8cm, text centered, drop shadow, 
      minimum height=1cm, 
      font=\scriptsize] (shared_G_box_anchor) at (3.6,3.6) 
      {\textbf{Shared Backbone Operator} $\mathcal{G}_\theta$: Params $\theta$ initialized \& refined across stages.};

\node[stage] (stage1) {
    \textbf{Stage 1: Depth $\ell=1$}\\[3pt]
    \textbf{Initialize }$\mathcal{O}^{(1)} := \mathcal{P} \circ \mathcal{G}_{\theta} \circ \mathcal{L}$\\[3pt]
    Train $\mathcal{O}^{(1)}$
    to obtain converged parameters $\theta^{(1)}$.
};

\node[stage, right=of stage1] (stage2) {
    \textbf{Stage 2: Depth $\ell=2$}\\[3pt]
    \textbf{Initialize }$\mathcal{O}^{(2)} := \mathcal{P} \circ (\mathcal{G}_{\theta^{(1)}} \circ)^2 \circ \mathcal{L}$\\[3pt]
    Retrain $\mathcal{O}^{(2)}$
    to obtain converged parameters $\theta^{(2)}$.
};

\node[right=0.6cm of stage2, font=\bfseries\Large] (dots) {\dots}; 

\node[stage, right=0.6cm of dots] (stagen) { 
    \textbf{Stage n: Depth $\ell=n$}\\[3pt]
    \textbf{Initialize }$\mathcal{O}^{(n)} := \mathcal{P} \circ (\mathcal{G}_{\theta^{(n-1)}} \circ)^n \circ \mathcal{L}$\\[3pt]
    Retrain $\mathcal{O}^{(n)}$
    to obtain converged parameters $\theta^{(n)}$.
};

\draw[arrow] (stage1.east) -- (stage2.west)
    node[midway, above, param_transfer, yshift=0.15cm] { $\theta^{(1)}$ }; 

\draw[arrow] (stage2.east) -- (dots.west)
    node[midway, above, param_transfer, yshift=0.15cm] {$\theta^{(2)}$}; 

\draw[arrow] (dots.east) -- (stagen.west)
    node[midway, above, param_transfer, yshift=0.15cm] { $\theta^{(n-1)}$ }; 

\draw[->, dashed, thin, draw=yellow!70!black, bend angle=10, bend left] (shared_G_box_anchor.south) to ($(stage1.north) + (0,0.05)$);
\draw[->, dashed, thin, draw=yellow!70!black] (shared_G_box_anchor.south) to ($(stage2.north) + (0,0.05)$);
\draw[->, dashed, thin, draw=yellow!70!black, bend angle=10, bend right] (shared_G_box_anchor.south) to ($(stagen.north) + (0,0.05)$);

\end{tikzpicture}
\caption{Illustration of the Train-and-Unroll (T\&U) strategy. The process starts with training a shallow model (Stage 1, depth $\ell=1$). Converged parameters $\theta^{(1)}$ of the shared backbone $\mathcal{G}_\theta$ are then used to initialize $\mathcal{G}_\theta$ for a deeper model (Stage 2, depth $\ell=2$), which is then trained. This incremental process of parameter transfer and retraining continues up to the desired depth $n$.}
\label{fig:train_and_unroll_diagram}
\end{figure}

\paragraph{Novelty and Practical Advantages}
The T\&U strategy uniquely adapts to the shared-parameter structure $\mathcal{P} \circ (\mathcal{G}_\theta \circ)^n \circ \mathcal{L}$, offering distinct benefits:
\begin{itemize}
	\item \textbf{Efficiency:} Training begins with inexpensive shallow models. Weight transfer accelerates convergence for deeper stages, reducing overall computational cost.
	\item \textbf{Adaptive Depth:} The process yields a sequence of usable models $\mathcal{O}^{(1)}, \dots, \mathcal{O}^{(n)}$, allowing early stopping once accuracy targets are met.
	\item \textbf{Stability:} Gradual depth increases act as a curriculum, stabilizing optimization compared to training full-depth models from scratch.
	\item \textbf{Extensibility:} Existing models can be easily extended to greater depths without restarting training.
\end{itemize}

\paragraph{Convergence Behavior}
Figure~\ref{fig:conv_curves} demonstrates the effectiveness and limitation of the Train-and-Unroll strategy. In the real-data ablation run shown in Figure~\ref{fig:conv_curves}(a), T\&U reaches lower loss than direct target-depth training while still providing usable intermediate models $O^{(3)}$ and $O^{(4)}$ before the final $O^{(5)}$ stage. Furthermore, the training error consistently decreases as the composition depth $\ell$ increases (Figure~\ref{fig:conv_curves}(b)), validating the theoretical error-decaying property of the self-composing architecture (Theorem~\ref{them:rate}) while also showing that the decay saturates above machine precision.

\begin{figure}[H] 
	\centering
	\subfloat[]{
	\includegraphics[width=0.55\linewidth]{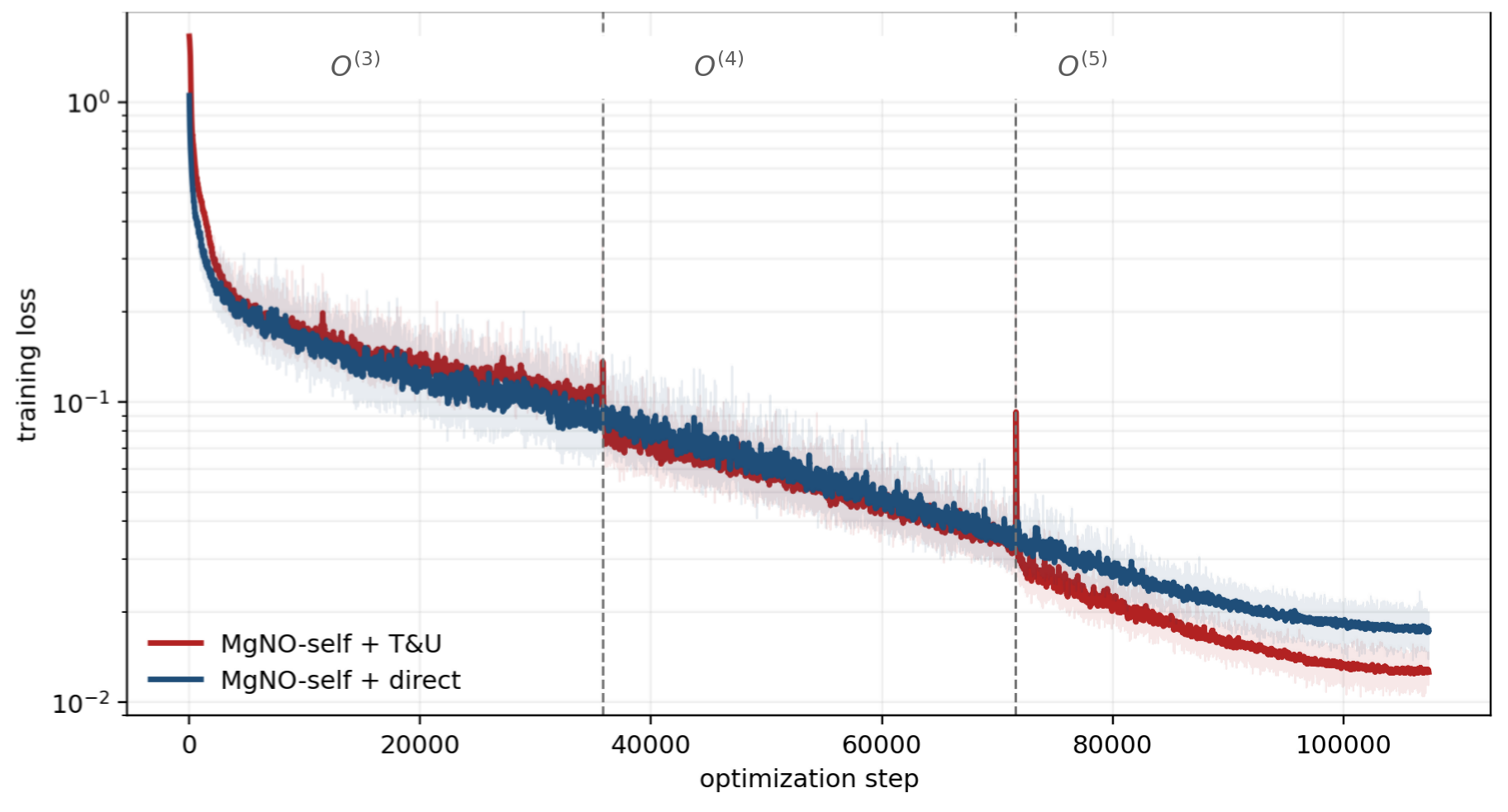}}  \quad 
	\subfloat[]{\includegraphics[width=.4\textwidth]{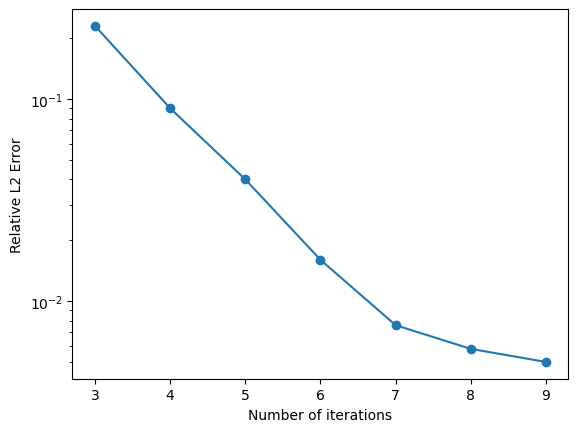}}
\caption{(a) Training-loss comparison between Train-and-Unroll (red) and direct training (blue) on the real-data USCT ablation run. T\&U follows the three-stage schedule $O^{(3)}\rightarrow O^{(4)}\rightarrow O^{(5)}$ with transitions after epochs 10 and 20, showing brief loss spikes at each depth increase followed by rapid recovery and lower final loss than direct target-depth training. (b) Decay of final training error with increasing depth $\ell$; the initial near-geometric decay eventually reaches a nonzero training-error floor rather than machine precision.}
	\label{fig:conv_curves}
\end{figure}
Figure~\ref{fig:conv_curves}(b) also clarifies a limitation of depth growth: increasing the number of self-compositions does not drive the surrogate to machine precision by itself. The error first decays nearly geometrically and then saturates above machine precision, which is consistent with the contractive analysis in Section~\ref{sec:exp-scaling}. Once the contraction term becomes smaller than the learned block bias or training-error floor, further compositions mainly recycle the approximation error of the trained backbone rather than eliminating it.
For reference, the updated real-data experiment in Figure~\ref{fig:conv_curves}(a) uses 1,790 training speed maps and 100 validation maps from OpenBreastUS. All compared MgNO variants are trained for 30 epochs with the relative $L^2$ loss, batch size 10, initial learning rate $3\times 10^{-4}$, and the same target depth $\ell=5$ on a single NVIDIA RTX A6000 GPU; T\&U allocates 10 epochs each to $O^{(3)}$, $O^{(4)}$, and $O^{(5)}$. In the completed run used for the wall-clock comparison, the staged T\&U schedule reduces total training time from 32.04 h for direct target-depth training to 23.96 h, a measured 25.2\% saving. Across three separate random initialization seeds under the same setup, direct training finishes with final validation loss $0.0194 \pm 0.0015$, whereas T\&U reaches $0.0151 \pm 0.0012$.


\section{An application to ultrasound computed tomography with specific backbone architecture}\label{sec:backbone}
In this section, we propose a specific architecture for the backbone operator $\mathcal G_\theta$ within the self-composing neural operator framework $\mathcal O (\bm v) = \mathcal P \circ (\mathcal G_\theta \circ)^n \circ \mathcal L (\bm v)$, tailored for solving the Helmholtz equation arising in ultrasound computed tomography (USCT).

\subsection{Operator learning task for Helmholtz equation}\label{sec:helmholtz_task}
In many physical problems, the input data consists of multiple fields, such as the wavenumber field $k(x)$ and the source term $f(x)$ in the Helmholtz equation. The neural operator should be able to learn the mapping from these multiple fields to the output field $u(x)$, which satisfies the governing equation.

Let $\Omega\subset\mathbb{R}^d$ ($d=2,3$) be a bounded Lipschitz domain with boundary $\Gamma$ and outward normal $n$.
Given a  variable wavenumber field $\,k\in L^\infty(\Omega)$ with
$0<k_{\min}\le k(x)\le k_{\max}<\infty$, and a boundary
wavenumber $k|_\Gamma>0$. Given a source   $f\in L^2(\Omega)$,
find $u\in H^1(\Omega)$ such that
\begin{equation}\label{eq:helmholtz-strong}
\begin{cases}
-\Delta u - k(x)^2\,u = f & \text{in }\Omega,\\[2pt]
\partial_n u - i\,k|_\Gamma\,u = 0 & \text{on }\Gamma.
\end{cases}
\end{equation}
For our numerical experiments on this task, we utilize the \textbf{OpenBreastUS} dataset, a comprehensive benchmark for wave imaging in breast ultrasound computed tomography recently introduced by Zeng et al. \cite{zeng2025openbreastus}. This dataset presents significant challenges due to the highly heterogeneous and multiscale nature of the wavenumber fields, which model different breast tissues. Accurately resolving the complex wave scattering and diffraction phenomena in such media is a difficult task for neural operators, making it an excellent testbed for evaluating model performance \cite{benitez2024out}.

We nevertheless do not interpret OpenBreastUS as a complete high-contrast generalization test. Larger sound-speed contrast can increase scattering, reflection, and multipath effects, and it should be evaluated directly in future stronger-contrast benchmarks. The 300--500 kHz multi-frequency setting used here is related but not equivalent: for a fixed medium, increasing frequency increases the wavenumber $k=\omega/c$ and therefore the phase accumulation and interference induced by the same relative sound-speed variation. Thus the higher-frequency USCT tests should be read as a frequency/scattering stress test within the OpenBreastUS distribution, not as evidence of robust extrapolation to arbitrary high-contrast media.

Following the self-composition paradigm, we learn a backbone operator $\mathcal{G}_\theta$ such that repeated application approximates the solution. Let $\vu, \vk, \vf$ be the discretized fields. The iteration takes the form:
\begin{equation}
\vu^i = \mathcal{G}_\theta(\vu^{i-1}, \vk, \vf), \quad i=1, \dots, n
\end{equation}
starting from $\vu^0 = 0$. 

\subsection{Multigrid-inspired backbone architecture}\label{sec:backbone_architecture}

For the Helmholtz problem, we instantiate $\mathcal{G}_\theta$ using a learnable multigrid V-cycle architecture \cite{he2019mgnet, he2023mgno}, chosen for its effectiveness in capturing multi-scale features and handling high-frequency errors.

The operator $\mathcal{G}_\theta$ updates a solution estimate $\vu$ using the wavenumber field $\vk$ and source $\vf$ via a residual update. Its core is a V-cycle with learnable components:
Implementation-wise, the USCT model first lifts the two-channel wave/source field $\theta$ and the one-channel sound-speed field into three 24-channel latent tensors through learned $1\times1$ convolutions, producing the multigrid state variables $u$, $f$, and $a$. The backbone then applies a six-level hierarchy on resolutions 480, 239, 119, 59, 29, 14, and 6. At each level, the local linear operators are implemented by learned $3\times3$ convolutions inside a dynamic update block: one pair of convolutions extracts a sound-speed-dependent gate from $a$, another convolution applies the residual-style operator update to $f$, and a final $3\times3$ convolution updates $u$. Coarsening to the next level is performed by learned stride-2 $3\times3$ convolutions on the latent solution, right-hand side, and coefficient channels, while prolongation is performed by learned transposed convolutions with level-dependent $3\times3$/$4\times4$ kernels chosen to match the non-power-of-two grid sizes. In the self-composing model used in the experiments, one initialization MG block is followed by a shared MG block that is reused across the remaining self-compositions, and the final two-channel wave field is obtained by a $1\times1$ projection together with an outer residual connection to the input field $\theta$.
\begin{itemize}
	\item \textbf{Grid Transfers:} Restriction $\mathcal{R}$ and prolongation $\mathcal{P}$ are implemented as strided and transposed convolutions, respectively.
	\item \textbf{Learnable Operators:} At each grid level $h$, we employ a learnable PDE operator $\mathcal{A}_h$ and smoother $\mathcal{S}_h$. Both are realized using the \emph{Adaptive Convolution Mechanism}(AdaConv) to modulates the features of a wave field $\vx$ based on local sound speed properties extracted from $\vk$.  It is defined as:
\begin{equation}
\label{eq:adaptive_conv}
\text{AdaConv}(\vk, \vx) = (\text{MLP}(\text{Filter}_k * \vk)) \odot (\text{Filter}_x * \vx),
\end{equation}
where $*$ denotes convolution and $\odot$ is element-wise multiplication. The term $\text{MLP}(\text{Filter}_k * \vk)$ acts as a spatial attention map, scaling the features of $\vx$ to model heterogeneous wave propagation effects. 
\end{itemize}


A single V-cycle update on grid $h$ proceeds as follows:
\begin{enumerate}
	\item \textbf{Pre-smoothing:} Compute the residual $\vr_h = \vf_h - \mathcal{A}_h (\vk_h,\vu_h)$ and update the solution $\vu_h \leftarrow \vu_h + \mathcal{S}_h(\vk_h, \vr_h)$.
	\item \textbf{Coarse-grid Correction:} Restrict the residual $\vr_{2h} = \mathcal{R} \vr_h$ and wavenumber $\vk_{2h}$ to the coarser grid. Recursively solve for the error correction $\ve_{2h}$.
	\item \textbf{Prolongation:} Interpolate the correction $\ve_h = \mathcal{P} \ve_{2h}$ and update the fine-grid solution $\vu_h \leftarrow \vu_h + \ve_h$.
	\item \textbf{Post-smoothing:} Apply the smoother $\mathcal{S}_h$ again to the updated solution to further reduce high-frequency errors.
\end{enumerate}
This recursive structure allows $\mathcal{G}_\theta$ to efficiently resolve errors across frequencies, with the adaptive components $\mathcal{A}_h$ and $\mathcal{S}_h$ capturing the local physics governed by $\vk$.

The overall self-composing construction and its USCT backbone are illustrated in Figure~\ref{fig:architecture}.

\begin{figure}
     \centering
     \includegraphics[width=.6\linewidth]{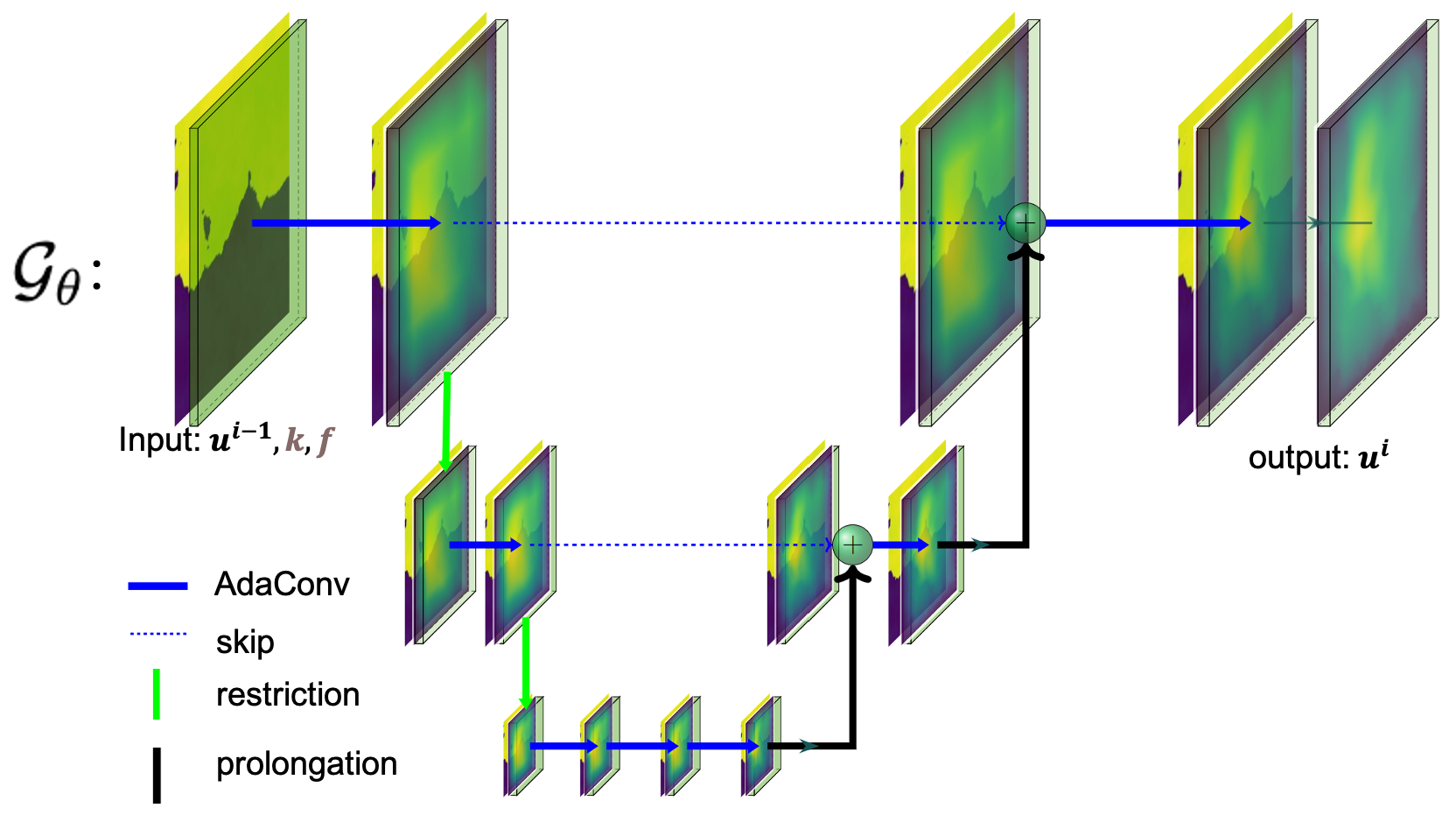}
     \caption{The diagram on the right illustrates the general architecture of the self-composing neural operator, expressed as $\mathcal P \circ (\mathcal G_\theta \circ)^n \circ \mathcal L$. On the left is the specific Multigrid-inspired backbone $\mathcal{G}_\theta$ designed for the Helmholtz application. 
	 A key feature is the Adaptive Convolution Mechanism (AdaConv), which distinguishes it from prior works such as \cite{he2019mgnet, he2023mgno}.}
     \label{fig:architecture}
\end{figure}

\subsection{Numerical experiments}\label{sec:numerics}

We evaluate MgNO-self on the USCT task, solving the Helmholtz problem \eqref{eq:helmholtz-strong} to map the wavenumber field $k$ and source $f$ to the wave field $u$. The model employs the multigrid-inspired backbone described in Section~\ref{sec:backbone_architecture}.

\begin{table}
    \scriptsize
    \centering
    \begin{tabular}{ccccccc}
    \hline
    \multirow{2}{*}{\textbf{Frequency (kHz)}} & \multicolumn{1}{c}{\multirow{2}{*}{\textbf{Metric}}}              & \multicolumn{5}{c}{\textbf{Models}}                                          \\ 
     & \multicolumn{1}{c}{}                                              & \textbf{UNet} & \textbf{FNO} &\textbf{AFNO}         & \textbf{BFNO}          & \textbf{MgNO-self}         \\ \hline
    \multirow{2}{*}{300}               & RRMSE     &  0.1236       &  0.0269      &  0.0165      & $ {0.0113}$   & \textbf{0.0028}       \\
   & Max Error      &  0.2551       &  0.0617      & $ {0.0293}$ &    0.0519              & \textbf{0.0092}       \\ 
   \hline
  \multirow{2}{*}{400}               & RRMSE     &  0.1503       &  0.0426      &  0.0242              &   $ {0.0148}$ & \textbf{0.0036}       \\
        & Max Error      &  0.3017       &  0.1172      & $ {0.0464}$ &    0.0840              & \textbf{0.0178}       \\ \hline
    \multirow{2}{*}{500}               & RRMSE       &  0.1798       &  0.0490      &  0.0276      &  $ {0.0209}$  & \textbf{0.0049}    \\
    & Max Error  &  0.3571  &  0.1432   & $ {0.0639}$ &    0.0838    & \textbf{0.0262}  \\ 
    \hline
    \multicolumn{2}{c}{\textbf{\# Parameters (M)}} & 36.0 & 734 & 58.6 & 104 & 26.6 \\
    \multicolumn{2}{c}{\textbf{Inference time (s)}} & 0.015 & 0.018 & 0.013 & 0.024 & 0.015 \\
      \hline
    \end{tabular}
    \caption{Quantitative evaluation on USCT forward simulation. Metrics: Relative Root Mean Square Error (RRMSE) and Maximum Error (Max Error). Models: UNet, FNO, AFNO\cite{guibas2021adaptive}, BFNO\cite{zhao2023deep}, and MgNO-self. Frequencies are 300, 400, and 500 kHz. Values are test-set means from the reported evaluation. The USCT task is also benchmarked in \cite{zeng2025openbreastus}, where MgNO-self is referred to as MgNO. The final two rows report model-size and single-forward-pass inference-time metadata for the corresponding forward surrogate models.}
    \label{tab:forward-baselines}
    \end{table}
    
    \begin{figure}[H]
        \centering
        \subfloat{\includegraphics[width=1\textwidth]{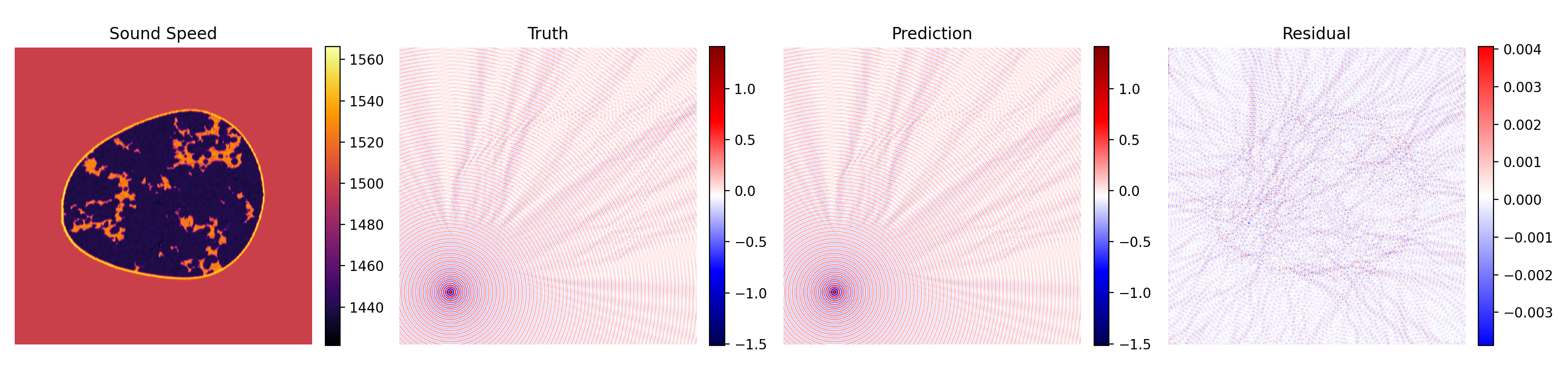}}\\[-0.5em]
        \subfloat{\includegraphics[width=1\textwidth]{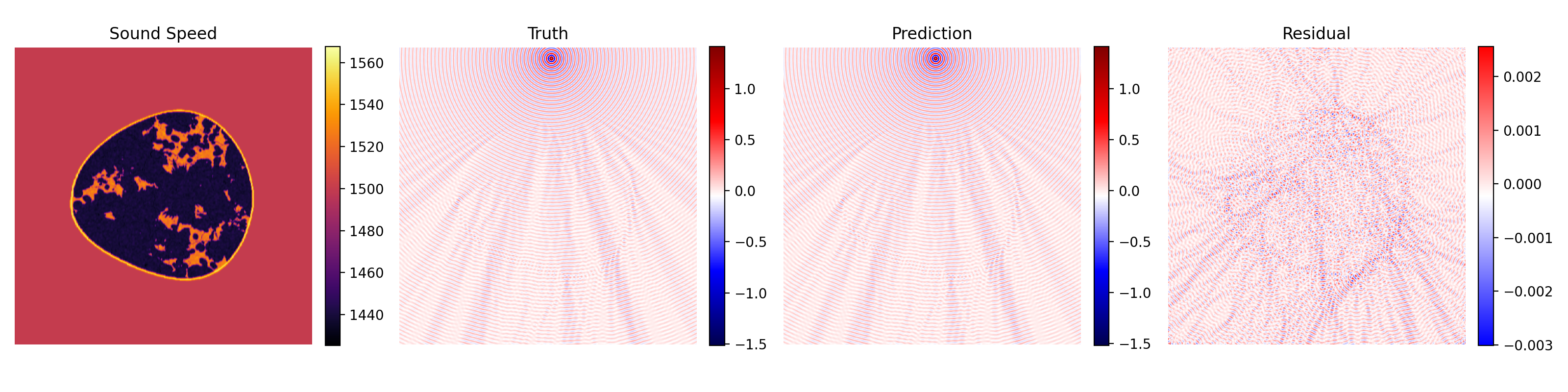}}\\
		\caption{Wave field solutions for different configurations. Each row represents a distinct test case with a specific sound speed distribution. Columns from left to right: input sound speed field $\vk$, ground-truth wavefield $\vu$, predicted wave field, and the  error $u_{\mathrm{pred}}-u$.} 
        \label{fig:forward-baselines}
    \end{figure}
To ensure reproducibility and establish a fair baseline comparison, all networks evaluated on the USCT task are trained under a consistent setup. The training set consists of 7,200 multiscale sound speed map and wave field pairs, and evaluation is performed on a held-out 900-sample test set at 300, 400, and 500 kHz. All models are trained for 30 epochs with the AdamW optimizer and batch size 10 on a single NVIDIA RTX A6000 GPU. Self-composed models use the Train-and-Unroll depth schedule, whereas the non-self baselines are trained directly at the target depth.

The parameter and inference-time rows in Table~\ref{tab:forward-baselines} provide cost metadata for the corresponding single-pass surrogate models. They should be read as surrogate-inference metadata, not as FLOP counts or as residual-driven solve times; a full hardware-normalized FLOP and wall-clock study remains separate from the present accuracy-focused comparison.

The errors in Table~\ref{tab:forward-baselines} are surrogate approximation errors, not residual tolerances from an online linear solve. If the discretized Helmholtz operator and residual are available during inference, shifted-Laplacian multigrid or shifted-Laplacian-preconditioned Krylov methods remain the appropriate classical solver-side references for high-accuracy single or few-solve regimes. Neural preconditioners that use learned blocks inside such residual-driven iterations are also an important residual-driven neural-preconditioning direction, and a full treatment of that setting is outside the scope of this surrogate-focused manuscript.


	Qualitative results are presented in Figure~\ref{fig:forward-baselines}, which illustrates the capability of the MgNO-self model to accurately predict wave field solutions for different input configurations. Each row in the figure typically displays the input sound speed field (related to the wave number $k$), the groundtruth wave field solution $\vu$ for a specific source term $\vf$, and the corresponding wave field solution predicted by MgNO-self. The visualizations demonstrate that the model effectively captures complex wave phenomena, including scattering effects due to inhomogeneous sound speed distributions, and accurately reproduces the wave field structure.

To isolate the effect of self-composition more directly, we additionally ran the small-scale 400 kHz ablation summarized in Table~\ref{tab:usct-small-ablation} on a 25\% USCT subset. This revision subset contains 1,790 training speed samples and 100 validation speed samples and uses a 30-epoch budget. We retain the comparison at this reduced scale because a fully crossed multi-backbone study on the full USCT set is computationally expensive, while the subset still provides a meaningful controlled comparison.

\begin{table}[H]
	\centering
	\scriptsize
	\begin{tabular}{llc}
		\toprule
		\textbf{Backbone} & \textbf{Variant} & \textbf{Val. RRMSE} \\
		\midrule
		FNO & regular & 0.0902 \\
		FNO & self & 0.0966 \\
		UNet & regular & 0.2559 \\
		UNet & self & 0.2610 \\
		MgNO & regular & 0.0194 \\
		MgNO & self & 0.0151 \\
		\bottomrule
	\end{tabular}
	\caption{Small-scale 400 kHz USCT ablation on the 25\% revision subset. The subset contains 1,790 training speed samples and 100 validation speed samples, and all runs are trained for 30 epochs.}
	\label{tab:usct-small-ablation}
\end{table}

Table~\ref{tab:forward-baselines} shows that the MgNO-based models achieve markedly lower RRMSE and Max Error than the generic UNet, FNO, AFNO, and BFNO baselines across 300--500 kHz, with about 9--13$\times$ RRMSE reductions relative to the spectral baselines. Table~\ref{tab:usct-small-ablation} then separates the backbone effect from the self-composition effect on the reduced 400 kHz USCT subset. For FNO and UNet, the regular and self-composed variants remain close, with the regular models slightly better in this short run (0.0902 versus 0.0966 for FNO and 0.2559 versus 0.2610 for UNet). The clearest gain appears in the multigrid-inspired backbone, where MgNO-self reaches 0.0151 compared with 0.0194 for regular MgNO.



\subsection{OT Car-CFD benchmark}\label{sec:carcfd_ot}

We additionally evaluated the OTNO family on the full Car-CFD OT surface dataset as a transfer test beyond the grid-based USCT setting. The input is a seven-channel feature field on an $84\times84$ OT latent lattice. OTNO follows the optimal-transport geometry encoding of Li et al.~\cite{li2025geometricot}: the surface is embedded on the OT lattice, operator updates are applied there, and a weighted sparse decoder maps the latent field back to the irregular surface mesh. In the OTMgNO variant, we keep the same OT lifting and weighted sparse decoder but replace the spectral OTNO blocks with MgNO-style multilevel blocks. The underlying Car-CFD mesh benchmark follows the aerodynamic-design dataset of Umetani and Bickel~\cite{umetani2018learning}. The target is the pressure field on 3,586 irregular surface points over 889 samples (778 training, 111 test). This remains a 3D surface-transfer task rather than a full volumetric solve, so we include it as supplementary evidence that the framework can extend beyond regular 2D grids.

\begin{table}[H]
	\centering
	\scriptsize
	\begin{tabular}{llc}
		\toprule
		\textbf{Backbone} & \textbf{Variant} & \textbf{Best val. relL2} \\
		\midrule
		OTNO & regular & 0.087868 \\
		OTNO & self & 0.089179 \\
		OTMgNO & regular & 0.076814 \\
		OTMgNO  & self & 0.078645 \\
		\bottomrule
	\end{tabular}
	\caption{Summary of the supplementary full Car-CFD OT benchmark. All runs use the same 778/111 train-test split of the 889-sample dataset. For each backbone, ``regular'' denotes the best untied configuration and ``self'' denotes the best shared configuration; for OTMgNO, the best shared result comes from a late-shared partial-tied regime.}
	\label{tab:carcfd_ot}
\end{table}

Table~\ref{tab:carcfd_ot} shows the main Car-CFD takeaway relevant to this manuscript. On the plain OTNO backbone, the best regular and self results remain close (0.087868 versus 0.089179), suggesting that sharing acts mainly as a mild regularizer. On the stronger OTMgNO backbone, both variants improve substantially, and the best self configuration remains competitive (0.078645) although the best regular configuration is still strongest (0.076814). Thus, the supplementary Car-CFD test supports the same measured conclusion as the USCT ablation: the framework extends beyond regular 2D grids, and the self-composition remains competitive while the gain is training efficiency.

\section{Conclusions and future work}\label{sec:conclusions}

We introduced self-composing neural operators as a compact way to build deeper PDE surrogates by repeatedly applying a shared backbone block. The construction is motivated by iterative numerical solvers, but it remains an amortized operator-learning method: after training, it maps problem data directly to the solution without running an online residual-driven solve. We established universal approximation and depth-dependent error-reduction results, and the experiments show that self-composition can improve parameter efficiency and accuracy on the tested Darcy and Helmholtz benchmarks. The Train-and-Unroll curriculum provides a practical way to train the repeated architecture by growing the active depth gradually.

The numerical evidence also suggests a measured interpretation. On USCT, the multigrid-inspired self-composed model substantially improves over the tested FNO-type baselines, while the supplementary Car-CFD study shows that the benefit of parameter sharing and Train-and-Unroll is backbone- and regime-dependent. The method should therefore be viewed as a useful architecture principle for many-query surrogate settings, not as a replacement for classical solvers or residual-driven neural preconditioners when only a few high-accuracy solves are needed. Important next steps include stronger high-contrast and out-of-distribution tests, closer integration with residual-driven neural preconditioning, and memory-control strategies such as activation checkpointing for deeper three-dimensional unrolls.

\section*{Acknowledgments}

J.H. was supported in part by the National Natural Science Foundation of China (Grant No. 12501606) and the Shandong Provincial Natural Science Foundation (Grant No. 2026HWYQ-024). X.L. was supported by the National Natural Science Foundation of China (Grant No. 12401571). J.X. was supported by King Abdullah University of Science and Technology (KAUST) Baseline Research Fund. The authors thank the DeepModeling community for hosting the ``AI for Science'' competition on the Bohrium platform. We also thank He Sun of Peking University for supporting the OpenBreastUS dataset used in this work.

\bibliographystyle{elsarticle-harv}
\bibliography{references,ref}
\end{document}